\documentclass{article}

\usepackage{booktabs} 

\usepackage{amsmath}
\usepackage{amsfonts}
\usepackage{amssymb}
\usepackage{amsthm}
\usepackage{color}
\usepackage{enumerate}
\usepackage{enumitem}
\usepackage{paralist}
\usepackage{footmisc}
\usepackage{amsbsy}
\usepackage{bm}
\usepackage{array,cases}
\usepackage{multirow}
\usepackage{algorithmic}
\usepackage{comment}
\usepackage{caption}
\usepackage{ifthen}
\usepackage{array}
\usepackage{times}
\usepackage{bbm}
\usepackage{graphicx}
\usepackage[hang]{subfigure}
\usepackage{epstopdf}
\usepackage{epsfig}
\usepackage[ruled]{algorithm2e}
\usepackage{balance}

\newboolean{longver}
\setboolean{longver}{true}

\newcommand{\set}[1]{\ensuremath{\mathcal #1}}

\newcommand{\separator}{
  \begin{center}
    \rule{\columnwidth}{0.3mm}
  \end{center}
}

 \def\11{{\textbf{1}}}

\def\Bl{\Bigl}
\def\Br{\Bigr}

\newcommand{\bprob}[1]{\mathbb{P}\Bl[ #1 \Br]}
\newcommand{\prob}[1]{\mathbb{P}[ #1 ]}

\newcommand{\beq}{\begin{eqnarray*}}
\newcommand{\eeq}{\end{eqnarray*}}
\newcommand{\beqn}{\begin{eqnarray}}
\newcommand{\eeqn}{\end{eqnarray}}
\newcommand{\bemn}{\begin{multiline}}
\newcommand{\eemn}{\end{multiline}}

\newtheorem{lemma}{Lemma}[section]

\newtheorem{theorem}{Theorem}

\newtheorem{corollary}{Corollary}

\newcommand\ie{{\em i.e.}}
\newcommand\eg{{\em e.g.}}
\newcommand\defeq{\mathrel{:=}}

\newcommand\etrue{e^t_{\text{\rm true}}}
\newcommand\Ttrue{T^t_{\text{\rm true}}}
\newcommand\Tttrue{T^{t+1}_{\text{\rm true}}}
\newcommand\diam{{\tt diam}}

\newcommand{\qnhatstar}{\hat{Q}^{\star}(n)}
\newcommand{\qnijhatstar}{\hat{Q}_{i,j}^{\star}(n)}
\newcommand{\enhatstar}{\hat{E}^{\star}(n)}
\newcommand{\qstar}{\hat{Q}^{\star}}
\newcommand{\estar}{\hat{E}^{\star}}

\title{Learning Data Dependency with Communication Cost}

\author{Hyeryung Jang$^\dag$, HyungSeok Song$^\dag$, and Yung
  Yi$^\dag$\thanks{$\dag$: Department of Electrical Engineering,
    KAIST, South Korea, e-mails: hrjang@lanada.kaist.ac.kr,
    hssong@lanada.kaist.ac.kr, yiyung@kaist.edu. Address for
    Correspondence: Yung Yi, KAIST 291, Daehak-ro, Yuseong-gu,
    Daejeon, 305-701, South Korea.}}

\date{}

\allowdisplaybreaks

\begin{document}
\maketitle

\begin{abstract}
  In this paper, we consider the problem of recovering a graph that
  represents the statistical data dependency among nodes for a
    set of data samples generated by nodes, which provides the basic
  structure to perform an inference task, such as MAP (maximum a
  posteriori). This problem is referred to as structure learning.
  When nodes are spatially separated in different locations, running
  an inference algorithm requires a non-negligible amount of message
  passing, incurring some communication cost. We inevitably have the
  trade-off between the accuracy of structure learning and the cost we
  need to pay to perform a given message-passing based inference task
  because the learnt edge structures of data dependency and physical
  connectivity graph are often highly different. In this paper, we
  formalize this trade-off in an optimization problem which outputs
  the data dependency graph that jointly considers learning accuracy
  and message-passing costs. We focus on a distributed MAP as the
  target inference task due to its popularity, and consider two
  different implementations, {\bf ASYNC-MAP} and {\bf SYNC-MAP} that
  have different message-passing mechanisms and thus different cost
  structures. In {\bf ASYNC-MAP}, we propose a polynomial time
  learning algorithm that is optimal, motivated by the problem of
  finding a maximum weight spanning tree. In {\bf SYNC-MAP}, we first
  prove that it is NP-hard and propose a greedy heuristic. For both
  implementations, we then quantify how the probability that the
  resulting data graphs from those learning algorithms differ from the
  ideal data graph decays as the number of data samples grows, using
  the large deviation principle, where the decaying rate is
  characterized by some topological structures of both original data
  dependency and physical connectivity graphs as well as the degree of
  the trade-off,
  which provides some guideline on how many samples are necessary to
  obtain a certain learning accuracy. We validate our theoretical
  findings through extensive simulations, which confirms that it has a
  good match.
\end{abstract}

\section{Introduction} \label{sec:intro}


In many online/offline systems with spatially-separated agents (or
nodes), a variety of applications involve distributed in-network
statistical inference tasks, which have been widely studied,
exploiting given knowledge of statistical dependencies among
agents. As one example, in sensor networks with multiple targets, each
sensor node measures the target-specific information in its coverage
area (\eg, position, direction, distance), which further has a
correlation among sensors. One well-recognized inference problem is a
data association which determines the correct match between
measurements of sensors and target tracks by maximum a posteriori
(MAP) estimation that is executed in a distributed fashion by
exchanging some information messages. Other examples include target
tracking, and detection/estimation in sensor networks
\cite{zhao2015energy, schiff2007robust, paskin2005robust,
  veeravali2012distinf} and de-anonymization, rumor/infection
propagation in social networks \cite{doost2014social, khan08kalman,
  ace11opinion, fu2017deanony}.

To solve these distributed in-network inference problems, it is of
crucial importance to understand how data from nodes are
inter-dependent.  To that end, a notion of the {\em graphical model}
has been one of the powerful frameworks in machine learning for a
succinct modeling of the statistical uncertainty, where each node in
the graphical models corresponds to a random variable and each edge
specifies the statistical dependency between random variables.
A wide variety of scalable inference algorithms on graphical models
via message-passing have been developed, of which examples include
belief propagation (BP) or max-product with a certain degree of
convergence and accuracy guarantees \cite{pearl14prob, globerson08map,
  wright05map, weiss01max}. This graphical model, which we also call
data dependency graph or simply data graph throughout this paper, is
not given a priori, and it should be learnt only by using a given set
of data samples from nodes.  This problem, referred to as {\em graph
  learning} or {\em structure learning} \cite{ravikumar2010l1,
  friedman2008lasso, abbeel06learning, meila00learning}, has been an
active research topic in statistical machine learning.

In this paper, for a collection of $n$ data sample vectors generated
by nodes, we study a problem of graph learning, which also considers
the communication cost incurred by the distributed in-network
inference algorithm being applied to the learnt data graph.  Physical
communication cost often becomes a critical issue, for example,
exerting a significant impact on the lifetime of networked
sensors. Clearly, there exists a trade-off between the amount of
incurred cost and the learning accuracy of the data
graph. Figure~\ref{fig:ex_phy_graph} illustrates the physical
connection among $7$ sensors, which differs from the exact data
dependency graph in Figure~\ref{fig:ex_data_graph}. The sensor nodes
$s1$ and $s6$ have non-negligible data dependency, requiring
message-passing when performing inference, but they are three hops
away from each other, incurring a large amount of communication
cost. In this case, one may want to sacrifice the estimation accuracy
a little bit and reduce communication cost by utilizing the data graph
as shown in Figure~\ref{fig:ex_opt_graph}. As done in many prior works
on graph learning \cite{meila00learning, chow1968tree,
  dasgupta1999poly, anandkumar2012mixture}, we restrict our attention
to tree-structured data graphs due to its simplicity, yet a large
degree of expressive powers and other benefits, \eg, some inference
algorithms such as BP over tree-structured data graphs become optimal.

We now summarize our contributions in what follows:

\begin{figure}[t!]
  \centering
  \hspace{-0.26cm}
  \subfigure[Physical graph]
  {
    \includegraphics[width=0.32\columnwidth]{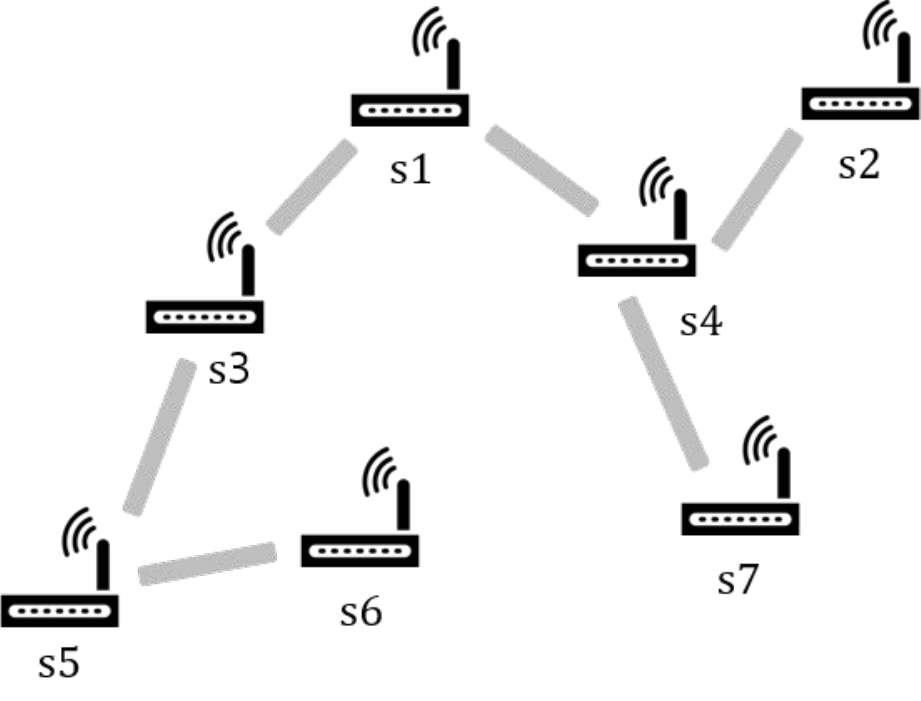}
    \label{fig:ex_phy_graph}
  } \hspace{-0.27cm}
  \subfigure[Data graph]
  {
    \includegraphics[width=0.32\columnwidth]{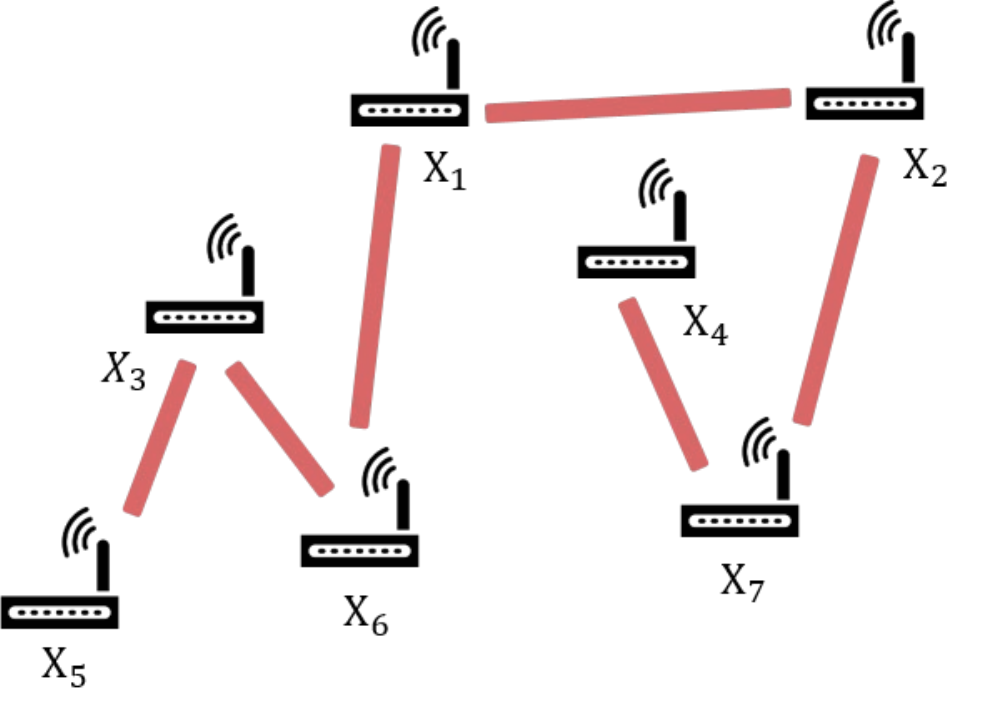}
    \label{fig:ex_data_graph}
  } \hspace{-0.27cm}
  \subfigure[Cost-efficient data graph]
  {
    \includegraphics[width=0.32\columnwidth]{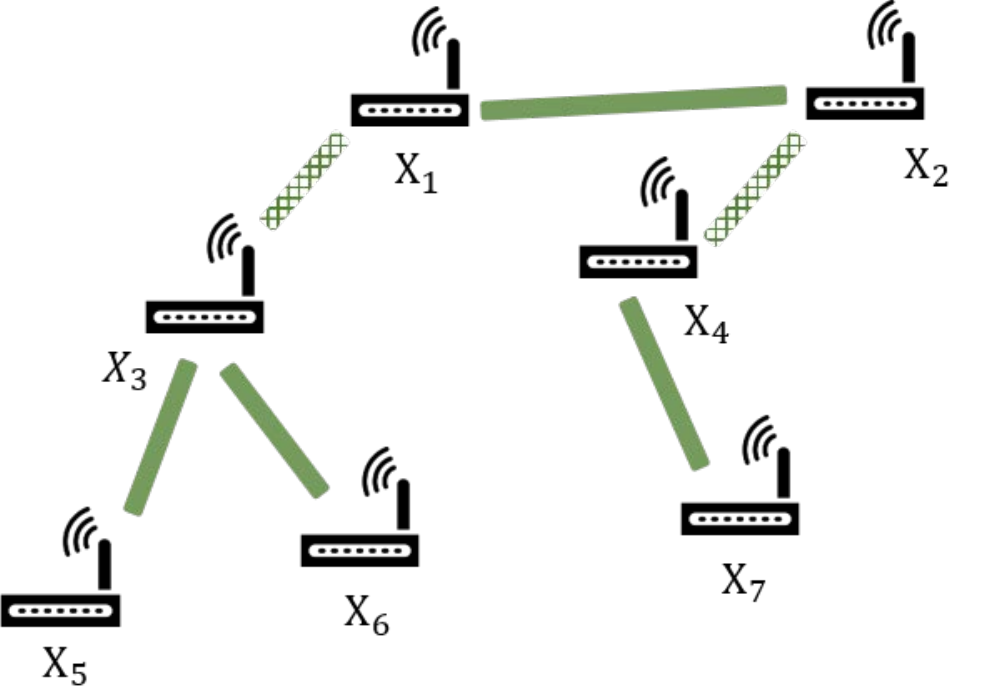}
    \label{fig:ex_opt_graph}
  }
  \hspace{-0.26cm}
  \caption{Network graphs with $7$ sensors. (a) physical
connectivity, (b) exact statistical dependency graph called data graph, (c) Data graph considering communication cost between  nodes.  \label{fig:ex_graph}}
\end{figure}

\smallskip
\begin{compactenum}[(a)]
\item We first formulate an optimization problem of learning data
  graph, having as an objective function the weighted sum of learning
  accuracy and the amount of cost that will be incurred by a
  distributed inference algorithm. Out of many possible inference
  algorithms, we consider the maximum a posteriori (MAP) estimator
  that is popular for many inference tasks, and two versions for the
  MAP implementation: (i) asynchronous and (ii) synchronous, which we
  call {\bf ASYNC-MAP} and {\bf SYNC-MAP}. These implementations have
  different patterns of passing messages, thus leading to different
  forms of communication costs, being useful to understand how
  distributed algorithms' cost affect the resulting data dependency
  graph.

\item Next, for {\bf ASYNC-MAP} we develop a polynomial-time algorithm
  to find an optimal (cost-efficient) data graph that corresponds to
  simply finding a maximum weight spanning tree. This simplicity stems
  from the cost structure of {\bf ASYNC-MAP} that is characterized
  only by the sum of all `localized' edge costs. Being in sharp
  contrast to {\bf ASYNC-MAP}, for {\bf SYNC-MAP} we first prove that
  it is computationally intractable (\ie, NP-hard) in terms of the
  number of nodes, by reducing it to the problem of the Exact Cover by
  $3$-sets. The hardness is due to the fact that the cost structure of
  {\bf SYNC-MAP} depends on the diameter of the resulting tree which
  is the `global' information involving the entire topology. As a
  practical solution, we propose a polynomial-time greedy heuristic to
  recover a sub-optimal, but cost-efficient data graph.
  
\item Finally, for both {\bf ASYNC-MAP} and {\bf SYNC-MAP}, we
  quantify how the probability that the resulting (cost-efficient)
  data graph for a finite number of $n$ samples differs from the ideal
  data graph decays as $n$ increases, using the large deviation
  principle (LDP), as a form of $\exp(-n \cdot K).$ The error exponent
  $K$ is characterized for each of {\bf ASYNC-MAP} and {\bf SYNC-MAP}
  by some topological information of physical/data graphs, cost
  structure for both inference mechanisms, and the degree of the
  trade-off. We validate our theoretical findings through simulations
  over a 20-node graph for a variety of scenarios and show their good
  match with the simulation results.
  
\end{compactenum}
\smallskip

To validate our theoretical results, we perform numerical simulations
a pair of physical and data graphs with 20 nodes, where we
quantitatively analyze (i) how estimating a data graph considering
communication cost affects the resulting estimation for various values
of trade-off parameters between inference accuracy and cost, (ii) how
the estimation error decays as the same size increases.

\subsection{Related Work}

A variety of applications which involve distributed in-network
statistical inference tasks among spatially inter-connected agents or
sensors have been widely studied in many online/offline systems. In
sensor networks, where the knowledge of statistical dependencies among
sensed data is given, the tasks of target tracking \cite{chen05mtt,
  cetin06fusion, xue08tracking}, detection \cite{chamber07detect},
parameter estimation \cite{ihler05bp, schiff2007robust} are the
examples, see \cite{veeravali2012distinf} for a survey. In social
networks, where the underlying social phenomenon of interest such as
voting models, rumor/opinion propagation \cite{ace11opinion} evolves
over a given social interaction graph, the inference tasks of
distributed consensus-based estimation \cite{khan08kalman},
de-anonymization of community-structured social network
\cite{fu2017deanony} and distributed observability
\cite{doost2014social} are studied.
Message-passing has manifested as an efficient procedure for inference
over graphical models that provide the framework of succinct model of
the statistical uncertainty of multi-agents. Examples include belief
propagation (BP) \cite{pearl14prob}, max-product \cite{weiss01max,
  globerson08map} and references therein.  They are known to be exact
and efficient when the underlying graphical model is a tree
\cite{pearl14prob, wright05map}. Recent research progress has been
made for scalable message-passing for general graphs, \eg, junction
tree \cite{wright03sum} and graphs with loops \cite{ihler05lbp}.
 
\smallskip
In the area of structure learning, several algorithms have been
proposed in the literatures to recover the statistical dependencies
from a set of data samples \cite{ravikumar2010l1, friedman2008lasso,
  abbeel06learning, meila00learning}. It is known that the exact
structure learning for general graphical models is NP-hard. The
research of structure learning for special graphical models includes:
maximum likelihood estimation (MLE) \cite{chow1968tree,
  meila00learning} for tree graphs, $\ell_1$ regularized MLE for
binary undirected graphs \cite{ravikumar2010l1}, convexified MLE for
Gaussian graphical models, known as Lasso \cite{friedman2008lasso}.
Theoretical guarantees for the learning accuracy have been established
as the number of data samples, \eg, on tree graph \cite{tan2011tree},
on binary undirected graphs \cite{ravikumar2010l1}, on a class of
Ising model \cite{bresler2015ising}, or on Bayesian network
\cite{zuk2012bay}. Our work differs from all of the above works in
that we consider physical communication cost incurred by some target
inference algorithms when learning the data dependency graph.

\smallskip

There exists an array of work that addresses the trade-off between
inference quality and cost in running distributed in-network inference
on the known data graph, which are summarized as two directions: {\em
  (i)} developing novel inference algorithms with less communication
of messages or {\em (ii)} constructing a new graphical model upon
which the existing distributed in-network inference algorithms are
performed with less communication resources. In {\em (i)}, the need of
conserving resources requires to propose new message-passing schemes
where the messages are compressed by allowing some approximation error
in message values \cite{cetin06fusion, kreidl06infcomm, ihler05lbp,
  dauwels07varmsg}, and/or some messages are censored (\ie, not to be
transmitted) \cite{chen05mtt}.
In {\em (ii)}, most of the related works focused on constructing a
junction tree that minimizes the inference cost
\cite{paskin2005robust}, building a data dependency structure upon
which message-passing is run energy-efficiently, where the
communications among all agents are assumed to be done in one-hop
\cite{zhao2015energy}, or optimizing the data dependency structure
formulated by a multi-objective problem of inference quality and
energy, assuming that the exact statistical dependencies are given as
a complete graph \cite{zhou2016secon}. While the main interest of this
area has been focused on characterizing the desirable dependency
structure for given complete knowledge of accurate data dependencies,
our work is motivated by the practical situation where one can just be
able to observe a finite number of data sample vectors of nodes, which
do not provide such a complete knowledge. Therefore, our interest lies
in {\em learning} the desirable data dependencies from a finite number
of data samples.


\section{Model and Preliminary} \label{sec:model-pre}

\subsection{Model} \label{sec:model}

\smallskip
\noindent{\bf \em Physical graph.} We consider a (connected) physical
network $G=(V,E_P)$ with a set of $d$ nodes $V$ and links $E_P$, where
each node corresponds to an agent such as a sensor or an individual,
and each link corresponds to a physical connectivity between two
nodes. For example, in sensor networks, when nodes have wireless
radios, then each link is established when two corresponding nodes
over the link reach each other within each radio's communication
range.

\smallskip
\noindent{\bf \em Data samples.} Each node $i \in V$ generates a
binary data, denoted by $x_i \in \set{X} \defeq \{0,1\}$\footnote{We
  assume a binary data for simplicity, and our results are readily
  extended to any finite set $\set{X}$.}, where we denote by
${\bm x} = [x_i]_{i \in V}$ the data vector of all nodes, or simply a
{\em sample}, \eg, target locations measured by all sensors. The
underlying statistical uncertainty of samples can be represented by a
joint distribution $P(\bm x)$ of a random vector
$X \defeq [X_i]_{i \in V} \in \set{X}^d$, called {\em data
  distribution,} where each random variable $X_i$ is associated to
each node $i \in V$. We often collect $n$ multiple samples
$({\bm x}^{1:n} = \{{\bm x}^1, {\bm x}^2, \ldots, {\bm x}^n\})$ in
order to infer what happens in the network by understanding the
inter-dependence of data generated by nodes. For instance, when
sensors measure the target location, then we infer the underlying
statistical correlation among sensors from the observed samples, to
estimate the true target location.


\smallskip
\noindent{\bf \em Data graph via graphical model.}
The underlying statistical dependency is often understood by the
framework of {\em graphical model}, which has been a popular tool for
modeling uncertainty by a graph structure, where each node corresponds
to a random variable and each edge captures the probabilistic
interaction between nodes. In particular, we model the data
distribution $P(\bm x)$ as an undirected graph $T = (V,E_D)$, which we
call {\em data graph}, which consists of the same set $V$ of nodes as
that in the physical graph and nodes' statistical dependencies
captured by an edge structure $E_D$ as: any two non-adjacent random
variables are conditionally independent given all other variables,
\ie, for any $(i,j) \notin E_D$,
\begin{align} \label{eq:pair-markov} 
P \left (x_i,x_j \mid x_{V \setminus \{i,j\}}\right) = P\left (x_i \mid x_{V \setminus \{i,j\}}\right ) \cdot P \left (x_j \mid
  x_{V \setminus \{i,j\}} \right ).
\end{align}

In this paper, we limit our focus on the tree-structured data graph
(thus simply data tree), for which let $\set{T}$ and
$\set{P}(\set{X}^d)$ be set of all spanning trees and set of all tree
data distributions over $V$, respectively,
\ie, we assume $T \in \set{T}$ and $P \in \set{P}(\set{X}^d)$.  Tree
data graph is a class of graphical models that has received
considerable attention in literatures \cite{chow1968tree,
  anandkumar2012mixture}, since it possesses the following
factorization property:
\begin{align} \label{eq:tree-dist}
P({\bm x}) = \prod_{i \in V}P_i(x_i) \prod_{(i,j) \in E_D} \frac{P_{i,j}(x_i, x_j)}{P_i(x_i)P_j(x_j)},
\end{align}
where $P_i$ and $P_{i,j}$ are the marginals on node $i \in V$ and edge
$(i,j) \in E_D$, respectively. Tree-structured data graph is known to
strike a good balance between the expressive power and the
computational tractability. In particular, the distribution $P$ in
\eqref{eq:tree-dist} is completely specified only by the set of edges
$E_D$ and their pairwise marginals. Thus, if $P$ has the factorization
property as in \eqref{eq:tree-dist}, in other words, if
$P \in \set{P}(\set{X}^d),$ there exists a unique tree $T=T(P)$
corresponding to $P.$ To abuse the notation, we henceforth denote by
$T(P)$ the unique data tree of a tree distribution $P.$
Figure~\ref{fig:ex_graph} shows an example of the physical graph and
two data graphs with $7$ nodes.

\subsection{Goal: Cost-efficient Learning of Data Graph}

\smallskip
\noindent {\bf \em Learning data graph: What and why?}
To understand the underlying data dependency \eqref{eq:tree-dist}, it
is enough to learn the structure of data graph $E_D$ from the observed
samples, which is known as the problem of {\em (data graph) structure
  learning.}  Formally, when we are given a set of i.i.d. $n$ samples
${\bm x}^{1:n}$ generated from an unknown (tree) data distribution
$P \in \set{P}(\set{X}^d)$ on a data tree $T$, a {\em structure
  learning} algorithm is a (possibly randomized) map $\phi$ defined
by:
\begin{align*}
  \phi: (\set{X}^d)^n \mapsto \set{T}.
\end{align*}
The quality of this algorithm $\hat{T} = \phi({\bm x}^{1:n})$ is
evaluated by how ``close" $\hat{T}$ is to the original data graph $T.$

\smallskip
\noindent {\bf \em Distributed inference on data graph.}
One of the practical goals of estimating the data tree given a set of
data samples is to perform an inference task based on $T.$ Thus, in
many applications, primary interests are not focused on data itself
but rather on how to exploit the data dependency for reliable decision
making, such as target tracking, detection, estimation in sensor
networks and/or social networks, which involves statistical inference
about the networks described by a data graph. One example of inference
tasks is the MAP (maximum a posteriori) based estimation.
Distributed in-network inference has been widely studied with the help
of various distributed algorithms on graphical models using {\em
  message-passing}. 
In particular, for a specific inference problem, a message between two
nodes contains the information on influence that one node exerts on
another, which is obtained based on the value contained in neighboring
messages over an estimated data graph $\hat{T}$. One critical issue of
message-passing based inference algorithm is that {\em messages are
  often passed along the multi-hop path on the physical graph $G$},
which incurs some amount of {\em communication cost}. Then, assuming
that some inference algorithm would be run for the estimated data
graph $\hat{T}$, such a data graph learning must have the trade-off
between the accuracy of the learnt graph (\ie, how close the learnt
graph is to the original data graph) and the communication cost
generated by performing the distributed inference.

\smallskip
\noindent{\bf \em Goal: Cost-efficient data graph learning.}
Given an observed samples ${\bm x}^{1:n}$ from the unknown data
distribution $P$, our objective is to estimate a cost-efficient data
tree, which captures the trade-off between {\em (i) inference
  accuracy} and {\em (ii) communication cost for inference}. For tree
distributions, finding a distribution naturally gives rise to the
corresponding data tree, as mentioned earlier. Thus, it is natural to
find the tree distribution
$ \qnhatstar = \hat{Q}^\star({\bm x}^{1:n}, G, \Pi, \gamma)$ that is
the solution of the following optimization problem:
for a constant parameter $\gamma \geq 0$ and a fixed inference
algorithm $\Pi,$
\begin{align} \label{eq:CDG-learn}
\textbf{CDG(${\bm n}$)}: & \qnhatstar = \underset{Q \in
                                  \set{P}(\set{X}^d)}{\arg \min} ~ D(\hat{P}({\bm
    x}^{1:n}) \parallel Q)+ \gamma C \big(T(Q);G,\Pi \big ),
\end{align}
where
$\hat{P}({\bm x}^{1:n}) \defeq \frac{1}{n} \sum_{k=1}^n
\mathbbm{1}\{{\bm x}^k = {\bm x}\}$ is the empirical distribution of
${\bm x}^{1:n}$, $D(\cdot \parallel \cdot)$ is some distance metric
between two distributions, and $C(T(Q);G,\Pi)$ is the communication
cost paid by running an inference algorithm $\Pi$ with respect to the
data tree $T(Q)$ over the physical graph $G.$ Recall that $T(Q)$ is
the data tree for the tree distribution $Q.$ The value of $\gamma$
parameterizes how much we prioritize the communication cost compared
to the inference accuracy $D(\hat{P}({\bm x}^{1:n}) \parallel Q).$
Note that as $n \rightarrow \infty,$ $\hat{P}({\bm x}^{1:n})$
converges to the original data distribution $P,$ which requires to
solve {\bf CDG($\infty$)}.

\smallskip
Then, this paper aims at answering the following two questions:
\smallskip
\begin{compactenum}[(a)]
\item What are good data-tree learning algorithms that compute
  $T(\qnhatstar)$ by solving {\bf CDG($\bm n$)}? In
  Section~\ref{sec:alg}, we consider the MAP estimator as an applied
  inference algorithm, and their two implementations having different
  cost functions, for which we propose two cost-efficient learning
  algorithms.
\item How fast does $\qnhatstar$ converge to $\qstar(\infty)$ as the
  number of samples $n$ grows? We use the large deviation principle
  (LDP) to characterize the decaying rate of the probability that
  $T(\qnhatstar) \neq T(\qstar(\infty))$ for two different MAP
  implementations in Section~\ref{sec:alg}.
\end{compactenum}

\smallskip

In this paper, we use the popular KL divergence as a distance metric
$D(\cdot || \cdot)$ for inference accuracy, denoted by
$D_{\text{KL}}$, where for two distributions $P$ and $Q$,
$D_{\text{KL}}({P} \parallel Q) := \sum_{{\bm x} \in \set{X}^d}
{P}({\bm x}) \log \frac{{P}({\bm x})}{Q({\bm x})}.$ For notional
simplicity, we simply denote by $Q^\star := \qstar(\infty)$ the
solution of {\bf CDG($\infty$)} throughout this paper.


\section{Cost-efficient Data Graph Learning
  Algorithms} \label{sec:alg}

In this paper, out of many possible inference tasks, we consider the
maximum a posteriori (MAP) estimation, which is popularly applied in
many applications such as data association for a multi-target tracking
problem in sensor networks, community-structured social network
de-anonymization problem in social networks \cite{fu2017deanony}.

\subsection{Distributed MAP and Cost} \label{sec:cost}

\smallskip
\noindent {\bf \em Distributed MAP on tree-structured data graph.}
The MAP estimator of some tree distribution $Q \in \set{P}(\set{X}^d)$
on its associated data tree $T(Q)=(V,E_Q)$ is given by:
\begin{align} \label{eq:map} {\bm x}^{\text{MAP}} \defeq
  \underset{{\bm x} \in \set{X}^d}{\arg\max} ~ \prod_{i \in V} \psi_i(x_i)
  \prod_{(i,j) \in E_Q} \psi_{i,j}(x_i,x_j),
\end{align}
where we use $\psi_i(x_i) = Q_i(x_i)$ and
$\psi_{i,j}(x_i,x_j) = \frac{Q_{i,j}(x_i,x_j)}{Q_i(x_i)Q_j(x_j)}$ for
simplicity. A standard message-passing algorithm for the distributed
MAP is a max-product algorithm, which defines a message
$m_{i \rightarrow j}^{(t)}(\cdot)$ from node $i$ to $j$ at $t$-th
iteration with $(i,j) \in E_Q$. Each node exchanges messages with
their neighbors on the data tree $T(Q)$, and these messages are
updated over time in an iterative fashion by the following rule: at
$t$-th iteration,
\begin{eqnarray} \label{eq:dist-msg} m_{i \rightarrow j}^{(t+1)}(x_j)
  \defeq \kappa \cdot \max_{x_i \in \set{X}} \Big[ \psi_i(x_i) \psi_{i,j}(x_i,x_j)
  \prod_{k \in \set{N}(i) \setminus \{j\}} m_{k \rightarrow i}^{(t)}(x_i)
  \Big],
\end{eqnarray}
with the normalizing constant $\kappa$ to make the sum of all message
values be $1$, and $\set{N}(i)$ denotes the neighboring nodes of $i$.

\smallskip
\noindent {\bf \em Communication cost of distributed MAP.}
The communication cost of MAP is paid, depending on the actual
protocol that specifies how to schedule message-passing
procedures. Two natural message-passing protocols studied in
literatures are: (a) asynchronous depth-first (unicast) update
\cite{elidan12residual} and (b) synchronous (broadcast) parallel
update \cite{pearl14prob}. Both protocols for a tree distribution $Q$
with its data tree $T(Q)$ have been shown to be consistent in that the
message update \eqref{eq:dist-msg} converges to a unique fixed point
$\{m^*_{i \rightarrow j}, m^*_{j \rightarrow i}\}_{(i,j) \in {E}_Q}$,
which defines the exact MAP assignment in \eqref{eq:map} as
$x^{\text{MAP}}_i = \kappa \cdot \psi_i(x_i) \prod_{k \in \set{N}(i)}
m_{k \rightarrow i}^*(x_i)$ for each $i \in V$. We denote the the cost
of a single message-passing over an edge $e = (i,j)$, under a given
physical graph $G$, as $c_e$ or $c_{i,j}$. Recall that the message
passing over $e=(i,j)$ may need to be done over a multi-hop path on
the physical graph $G.$ One simple example of $c_{i,j}$ is the
shortest path distance from node $i$ to $j$ in $G$. Then, both
protocols incur the communication cost as elaborated in what follows:

\smallskip
\begin{compactenum}[(a)]
\item {\em Asynchronous}: In the {\em asynchronous} protocol (simply
  {\bf ASYNC-MAP}), one node is arbitrarily picked as a root, and
  messages are passed from the leaves upwards to the root, then back
  downwards to the leaves. It involves a total $|E_Q|$ number of
  messages upon termination. Thus, the communication cost would be:
  \begin{align} \label{eq:ccost-async} C(T(Q); G,\textbf{ASYNC-MAP}) =
    \sum_{(i,j) \in E_Q} 2 c_{i,j}.
  \end{align}

\item {\em Synchronous:} In the {\em synchronous} protocol (simply
  {\bf SYNC-MAP}), at each iteration, every node sends messages to all
  of its neighbors. Then, since the diameter
  $\diam(T(Q))$\footnote{For a tree $T$ with $d$ nodes,
    $2 \le \diam(T) \le d-1.$} is the minimum amount of time required
  for a message to pass between two most distant nodes in $T(Q)$, this
  protocol involves at most $\diam(T(Q))$ iterations with total
  $2 |E_Q| \cdot \diam(T(Q))$ number of messages. Thus, we have the
  following cost:
  \begin{align} \label{eq:ccost-sync} C(T(Q); G,\textbf{SYNC-MAP}) =
    \sum_{(i,j) \in E_Q} 2 c_{i,j} \cdot \diam(T(Q)).
  \end{align}
\end{compactenum}

\smallskip

In the next subsection, we will use the above two cost functions for
two different learning algorithms for $\textbf{CDG(${\bm n}$)}$ in
\eqref{eq:CDG-learn} to estimate two cost-efficient data trees.

\subsection{Algorithm for Asynchronous MAP} \label{sec:async-alg}

Using the cost function for the asynchronous MAP in
\eqref{eq:ccost-async}, the original optimization problem
$\textbf{CDG(${\bm n}$)}$ is re-cast into:
\begin{align} \label{eq:CDG-learn-async} 
&\textbf{CDG-A(${\bm n}$)}:  \ \qnhatstar=\cr  
& \underset{Q \in \set{P}(\set{X}^d)}{\arg \min} ~
  D_{\text{KL}}(\hat{P}({\bm x}^{1:n}) \parallel Q) + \gamma
  \sum_{e \in E_Q} 2c_e.
\end{align}

We now describe {\bf ASYNC-ALGO} that computes $\qnhatstar$ in
\eqref{eq:CDG-learn-async} and thus estimates the cost-efficient data
tree $T(\qnhatstar)$ in Algorithm~\ref{alg:alg-async}.  As we see, the
algorithm is remarkably simple. Using given $n$ data samples, we
construct a weighted complete graph, where the weight for each edge is
assigned some combination of the mutual information of nodes $i$ and
$j$ with respect to the empirical distribution $\hat{P}$ obtained from
the data samples and the per-message cost, as in
\eqref{eq:weight-async}. Then, we run an algorithm that computes the
maximum weight spanning tree, \eg, Prim's algorithm or Kruskal's
algorithm, and the resulting spanning tree is the output of this
algorithm.

\DecMargin{1em}
\begin{algorithm}[t!]
  \caption{ASYNC-ALGO}
  \label{alg:alg-async}

  \KwIn{${\bm x}^{1:n}$: a set of $n$ samples, $\gamma$: the trade-off
    parameter, a physical graph $G=(V,E_P)$}
  
  \KwOut{Estimated tree $T = (V,E).$ }

  \medskip
  \hrule
  \medskip
  \begin{compactitem}
\item[\bf S0.]   $E = \emptyset$ and 
    for each possible edge $e \in V \times V,$ we initialize its
    weight by:
  \begin{align} \label{eq:weight-async} w_e= I_e(\hat{P})-
    2\gamma \cdot c_e,
    \end{align}
    where $I_e(\mu)$ is the mutual information between two end-points
    of edge $e$ with respect to a given joint distribution $\mu.$
\smallskip
  \item[\bf S1.] Run a maximum weight spanning tree algorithm for $H,$
    and save its resulting spanning tree at $T = (V,E)$.

\smallskip
  \item[\bf S2.] Return $T.$
  \end{compactitem}
\smallskip
\end{algorithm}
\IncMargin{1em}

\smallskip
\noindent {\bf \em Correctness of ASYNC-ALGO.} We now present the
correctness of the above algorithm in the sense that we can obtain the
data tree corresponding to the optimal distribution formulated in
\eqref{eq:CDG-learn-async}, as explained in what follows: For some
tree distribution $Q$ (thus, satisfying the factorization property in
\eqref{eq:tree-dist}), we have:
\begin{align} \label{eq:tree-prop1} D_{\text{KL}}(\hat{P} \parallel Q)
  &= -H(\hat{P}) - \sum_{{\bm x} \in \set{X}^d} \hat{P}({\bm x}) \log
  Q({\bm x}) \cr & \geq -H(\hat{P}) + \sum_{i \in {V}}
  H(\hat{P}_i) - \sum_{(i,j) \in {E}_Q} I(\hat{P}_{i,j}),
\end{align}
where $H(\cdot)$ is the entropy, and the inequality holds when the
pairwise marginals over the edges of a fixed $E_Q$ are set to that of
$\hat{P}$, \ie, $Q_{i,j}(x_i,x_j) = \hat{P}_{i,j}(x_i,x_j)$ for all
$(i,j) \in E_Q$. Since the entropy terms are constant w.r.t.  $Q$, it
is straightforward that the structure of the estimator $\qnhatstar$ of
{\bf CDG-A}$({\bm n})$ in \eqref{eq:CDG-learn-async} is given by:
\begin{align} 
\enhatstar &:= \underset{E_Q: Q \in \set{P}(\set{X}^d)}{\arg\max}
  \sum_{e \in E_Q} I_e(\hat{P}) - 2\gamma \cdot c_e,  \label{eq:mwst-async} \\
    \qnijhatstar &= \hat{P}_{i,j}, \quad  \forall (i,j) \in \enhatstar.
\end{align}
Then, it is easy to see that \eqref{eq:mwst-async} requires us to find
the maximum weight spanning tree using
$I_e(\hat{P}) - 2\gamma \cdot c_e$ as the edge $e$'s weight, where the
standard maximum weight spanning tree (MWST) computation algorithm
runs in $O(d^2 \log d)$ time, where recall that $|V| =d.$

\subsection{Algorithm for Synchronous
  MAP} \label{sec:sync-alg}

Similarly to {\bf ASYNC-MAP}, using the cost in \eqref{eq:ccost-sync},
the original optimization problem $\textbf{CDG(${\bm n}$)}$ is re-cast
into:
\begin{align} \label{eq:CDG-learn-sync} 
&\textbf{CDG-S(${\bm n}$)}: \qnhatstar = \cr
& \underset{Q \in \set{P}(\set{X}^d)}{\min} ~
  D_{\text{KL}}(\hat{P}({\bm x}^{1:n}) \parallel Q) + \gamma \cdot \diam(T(Q))
  \sum_{e \in E_Q} 2c_{e}.
\end{align}

Following the similar arguments in Section~\ref{sec:async-alg}, the
structure of the above estimator of {\bf CDG-S}$({\bm n})$ in
\eqref{eq:CDG-learn-sync} is given by

\begin{align} 
\enhatstar &:= \underset{E_Q: Q \in \set{P}(\set{X}^d)}{\arg\max}
  \sum_{e \in E_Q} I_e(\hat{P}) - 2\gamma \diam(T(Q))\cdot
                  c_e,  \label{eq:mwst-sync}\\
    \qnijhatstar &= \hat{P}_{i,j}, \quad  \forall (i,j) \in \enhatstar.
\end{align}
We comment that this optimization is non-trivial in that the objective
function contains the diameter of the tree, which can be computed only
when the solution is fully characterized.

\smallskip
\noindent{\bf \em Hardness.} The key difference in the cost function
of {\bf SYNC-MAP} from {\bf ASYNC-MAP} is simply the existence of
$\diam (T(Q))$. However, this simple difference completely changes the
hardness of learning the optimal data tree in {\bf SYNC-MAP}, as
formally stated in the next Theorem.

\begin{theorem}[Hardness of {\bf CDG-S}$({\bm
    n})$] \label{thm:sync-np} For any parameter $\gamma \geq 0$,
  obtaining the optimal distribution $\qnhatstar$ in {\bf
    CDG-S}$({\bm n})$ and thus its associated data tree
  $T(\qnhatstar)$ is NP-hard with respect to the number of nodes.
\end{theorem}

\smallskip 
\noindent{\bf \em Proof sketch.} Due to space limitation, we present
the full proof of Theorem~\ref{thm:sync-np} in
Appendix~\ref{sec:proof-np}, and we only provide its sketch here. The
key step in proof is to reduce the {\bf CDG-S}$({\bm n})$ in
\eqref{eq:CDG-learn-sync} to the well-known NP-complete problem: {\em
  Exact Cover by $3$-sets} problem, which we simply call {\bf X3C}. In
\cite{garey02computers}, the bounded diameter minimum weight spanning
tree (BDMST) problem that finds the MWST with a diameter less than $k$
of $4 \leq k \leq |V|-2$ for a fixed edge weights is shown to be an
NP-hard problem, by reducing it to the {\bf X3C} problem. The main
technical challenge in {\bf CDG-S}$({\bm n})$ lies in that the edge
weights are {\em diameter-dependent}, via the form of
$I_e(\hat{P}) - 2\gamma \cdot \diam(T(Q)) \cdot c_e$ in
\eqref{eq:mwst-sync}, where the weights become smaller as the diameter
$\diam(T(Q))$ grows. Therefore, the optimal structure of $\enhatstar$
in \eqref{eq:mwst-sync} would be attained at the tree with small
diameter. If we consider a fixed diameter $\diam(T(Q)) = k$ of
$4 \leq k \leq |V|-2$ so that the edge weights are set by constant
values, then the problem becomes similar to the BDMST problem. To
prove NP-hardness of our problem, we first construct a specific tree
distribution $\bar{P}(\bm x)$ and the cost functions
$\{\bar{c}_{i,j}\}_{(i,j) \in V \times V}$, at which the optimal
solution of {\bf CDG-S}$(\bm n)$ should have a certain diameter, a
diameter of $4$ in our proof, then we show that {\bf CDG-S}$(\bm n)$
for the weights of diameter $4$ has the optimal solution with diameter
$4$ if and only if {\bf X3C} problem has a solution. From
understanding the reduction of BDMST problem, we construct
$\bar{P}(\bm x)$ and $\{\bar{c}_{i,j}\}_{(i,j) \in V \times V}$, under
which (i) the tree with diameter less than $3$ does not attain optimal
solution due to its structural limitations (to force the small
diameter), and (ii) the edge weights for the diameter larger than $5$
become too small to achieve optimal solution of {\bf CDG-S}$(\bm
n)$. The remaining technique to verify the reduction of our problem to
the {\bf X3C} problem follows the arguments in
\cite{garey02computers}. Then, we are done with the reduction.

\DecMargin{1em}
\begin{algorithm}[t!]
  \caption{SYNC-ALGO($\beta$)}
  \label{alg:alg-sync}

  \KwIn{${\bm x}^{1:n}$: a set of $n$ samples from $P$, $\gamma$: the
    trade-off parameter, a physical graph $G=(V,E_P),$ and a tunable
    parameter $\beta.$}
  
  \KwOut{Estimated tree $T = (V_S,E_S)$.}

  \medskip
  \hrule
  \medskip
  \begin{compactenum}
  \item[\bf S0.] $V_S = \emptyset, E_S = \emptyset,$ and
    for each possible edge $e \in V \times V,$ we initialize its
    weight by:
  \begin{align} \label{eq:weight-sync-init} w_e= I_e(\hat{P})-
    2\gamma \cdot c_e,
  \end{align}
  and initialize the edge set $E'$ by the set of all possible edges. 
  \end{compactenum}
\smallskip
  \Repeat{$V_S = V$}{

    \begin{compactenum}[\bf S1.]
\smallskip
\item Select an edge $e=(u,v) \in E'$ with the maximum weight, and
      update $V_S \leftarrow V_S \cup \{u,v\}$ and
      $E_S \leftarrow E_S \cup \{e\}$.

\smallskip
    \item Update $E'$ as the set of all edges $e=(i,j),$ such that
      $i \in V_S$ and $j \in V \setminus V_S,$ and set the weight of
      each edge $e=(i,j) \in E'$ as:
      \begin{align} \label{eq:weight-sync} w_e &= I_e(\hat{P}) -
        2\gamma \diam(T \cup \{e \}) \cdot c_e \cr & - \beta
        \frac{d}{|E_S|} \cdot K_e(T) - 2\gamma \cdot D(T) \cdot c_e,
      \end{align}
      where
      \begin{align} \label{eq:inc-old} K_e(T) &= \Big( \diam(T\cup \{e
        \}) - \diam(T) \Big) \cdot \sum_{e' \in E_S} 2\gamma \cdot
        c_{e'},
      \end{align}
      and
      \begin{align} \label{eq:diam-exp} D(T) = \sqrt{d} \cdot \left(
          1-\frac{\sqrt{|E_S|}}{\sqrt{d}} \right).
      \end{align}

    \end{compactenum}
  }
  \smallskip
  Return $T=(V_S,E_S)$.
\smallskip
\end{algorithm}
\IncMargin{1em}

\smallskip
\noindent{\bf \em Greedy algorithm.} Due to the above-mentioned
hardness, we propose a greedy heuristic algorithm that outputs the
tree structure denoted by $\hat{E}^\text{S}(n)$, called {\bf
  SYNC-ALGO($\beta$)}, as we describe in Algorithm~\ref{alg:alg-sync},
where $\beta$ is the algorithm parameter.  The overall algorithm
operates as follows:

\smallskip
\begin{compactenum}
\item[\bf S0.] Initialize the weight of each possible edge with some
  initial value.

\item[\bf S1.] Sequentially select the edge that has the maximum
  weight and add it to the temporary resulting tree.

\item[\bf S2.] Update the weight of each edge whose one end-point is
  in the current resulting tree $V_S$ and another end-point is not,
  and go to {\bf S1} until we handle all nodes.

\end{compactenum}

\smallskip

One of the central steps here is: first, we {\em dynamically} update
the weight of the candidate edges (\ie, the set $E'$) that we will add
and, second, which value is chosen as the weight is different from the
``one-shot'' weight assignment as done in {\bf ASYNC-ALGO}.  To
explain this intuition, we first note that from \eqref{eq:mwst-sync}
it is easy to see that the degree of contribution in terms of weight
by adding an edge $e \in E'$ to the existing resulting tree would be
re-expressed as:
\begin{align} \label{eq:origin}
  I_e(\hat{P}) - 2\gamma \diam(T \cup \{e\})) \cdot c_e - K_e(T),
\end{align}
where $K_e(T)$ is defined in \eqref{eq:inc-old}. Here, $K_e(T)$
corresponds to the change of the communication cost over the existing
edges in $T$, under the grown tree $T \cup \{ e\}.$ For example,
$K_e(T) = 0$ if the diameter of the grown tree does not change by
adding the edge $e,$ or
$K_e(T) = \sum_{e' \in E_S} 2\gamma \cdot c_{e'}$, if the diameter of
the grown tree increases by $1$.

In dynamically assigning the weight of the candidate edges in $E',$ we
do not use the value of \eqref{eq:origin}. Instead, as seen in
\eqref{eq:weight-sync}, (i) we use the expected diameter growth of the
tree, denoted by $D(T)$ in \eqref{eq:diam-exp}, and (ii) we use a
tunable parameter $\beta > 0$ to compensate for the impact of the
change in communication cost over the existing edges $K_e(T)$ in
\eqref{eq:inc-old}. In more detail, we use $D(T)$ in
\eqref{eq:diam-exp}, which captures the expected diameter growth of
the tree $T$ via the term
$\sqrt{d} (1 - \frac{\sqrt{|E_S|}}{\sqrt{d}})$, since the diameter of
a uniformly random spanning tree is known to be of the order
$\sqrt{d}$ in \cite{renyi67height}. We note that this term decreases
to $0$ as the tree becomes to a spanning tree from the term
$1- \frac{\sqrt{|E_S|}}{\sqrt{d}}$. Second, we consider the impact of
old weights over the existing edges in $T$, captured by $K_e(T)$ in
\eqref{eq:origin}, by controlling a scale of $\beta \frac{d}{|E_S|}$.

\smallskip
To summarize, these two modified choices of the weight are for
handling a probable sacrifice of the performance when using a vanilla
greedy method as in \eqref{eq:inc-old}, since the edge weight should
be modified suitably for the changed diameter on the way of tree
construction. We expect that these two engineerings play an important
role when the cost-efficient data graph is attained with a large
diameter, where the edges chosen in the begging phase of the procedure
(\ie, with a small diameter value) could exert much impact of
communication cost at the end of the procedure. Our greedy algorithm
runs in $O(d^4)$ times.


\section{Estimation Error for Increasing Sample
  Size} \label{sec:analysis}

In this section, we provide the analysis of how the estimation error
probability decays with the growing number of samples $n,$ using the
large deviation principle (LDP).

\subsection{Estimation Error of ASYNC-ALGO} \label{sec:async-limit}

Clearly, when we use more and more data samples, $\enhatstar$
approaches to $\estar(\infty)$ that is the optimal edge structure
solving {\bf CDG-A}$({\bm \infty}).$ We are interested in
characterizing the following error probability of the event
$\set{A}_n$:
\begin{align} \label{eq:eevent-async} 
  \bprob{\set{A}_n(\bm x^{1:n}) := \big\{ \enhatstar \neq \estar(\infty) \big\} }. 
\end{align}
To characterize the probability in \eqref{eq:eevent-async} that is one
of the rare events, we use LDP that {\em rare events occurs in the
  most probable way.} To this end, we aim at studying the following
rate function $K=K(\gamma)$:
\begin{align} \label{eq:eexponent-async1} K(\gamma) \defeq \lim_{n
    \rightarrow \infty} - \frac{1}{n} \log \mathbb{P}(\set{A}_n(\bm
  x^{1:n})),
\end{align}
whenever the limit exists.

We now consider a simple event, called {\em crossover event}, as
defined in what follows:
Recall that {\bf ASYNC-ALGO} uses, for each edge $e$, the
weight\footnote{We interchangeably use $w_e(\hat{P})$ to denote the
  assigned weight of an edge $e$ in algorithms, with respect to the
  empirical distribution $\hat{P}$ from the given samples.} of
$w_e(\hat{P}) = I_e(\hat{P}) - 2\gamma c_e$ based on the empirical
distribution $\hat{P}.$ Then, consider two edges $e$ and $e'$ such
that the weight of $e$ exceeds that of $e'$ with respect to the {\em
  true} distribution $P,$ \ie, $w_e(P) > w_{e'}(P)$. We now define the
crossover event for two edges $e$ and $e'$ as:
\begin{align} \label{eq:cevent-async} 
  C_{n}(e,e') \defeq \big\{ w_e(\hat{P}) \leq w_{e'}(\hat{P}) \big\}.
\end{align}
As the number of samples $n \rightarrow \infty$, the empirical
distribution approaches to the true distribution, thus the probability
of the crossover event decays to zero, whose decaying rate which we
call {\em crossover rate} is defined as
$J_{e,e'} :=\lim_{n \rightarrow \infty} -\frac{1}{n}\log
\prob{C_n(e,e')}$. Using this definition of the crossover event, we
present Theorem~\ref{thm:eexponent-async} that states the decaying
rate of the estimation error probability as the number of data samples
$n$ grows.

\begin{theorem}[Decaying rate of {\bf ASYNC-ALGO}]
\label{thm:eexponent-async} 
For any fixed parameter $\gamma \geq 0$,
\begin{align}
  \lim_{n \rightarrow \infty} - \frac{1}{n} \log
  \mathbb{P}(\set{A}_n(\bm x^{1:n})) &= K(\gamma),
\end{align}
where
\begin{align} \label{eq:eexponent-async} K(\gamma) = \min_{e' \notin
    \estar(\infty)} ~ \min_{e \in {\Psi}(e'; \estar(\infty))} ~ J_{e,e'},
    \end{align}
    where
    $\Psi(e'=(i,j); \estar(\infty)) \defeq \{ v_1(=i), v_2, \cdots,
    v_l(=j) \}$ is the unique path between nodes $i$ and $j$, such
    that $(v_k, v_{k+1}) \in \estar(\infty)$ for
    $1 \leq k \leq l-1$,
    and
    \begin{eqnarray} \label{eq:crate-async}
      J_{e,e'} = \left.
      \begin{cases}
        \underset{Q \in \set{P}(\set{X}^4)}{\inf} & \Big\{
        D_{\text{KL}}(Q \parallel P_{e,e'}) : w_e(Q) = w_{e'}(Q)
        \Big\}, \cr & \text{if} ~ \left\{Q \in \set{P}(\set{X}^4):
        w_e(Q) = w_{e'}(Q) \right\} \neq \emptyset, \cr \quad \infty, &
        \text{otherwise.}
      \end{cases}
       \right.
    \end{eqnarray}
    Moreover, we have the following (finite-sample) upper-bound on the
    error probability: for all $n = 1, 2, \ldots, $
    \begin{align} \label{eq:sbound-async}
      \bprob{\set{A}_n({\bm x}^{1:n})} \leq \frac{(d-1)^2(d-2)}{2}\binom{n-1+|\mathcal{X}|^{4}}{|\set{X}|^4-1}
      \exp(-n \cdot K(\gamma)).
    \end{align}
\end{theorem}

In Theorem~\ref{thm:eexponent-async}, we observe that the decaying
rate of error probability is specified by some topological information
of physical/data graphs and the trade-off parameter $\gamma$. In
particular, the crossover event and its rate $J_{e,e'}$ depend on how
difficult it is to differentiate two edge weights under the true data
distribution with a consideration of the trade-off parameter $\gamma$
as well as per-message cost on edges. As interpreted from
\eqref{eq:crate-async}, when $w_e(P) = I_e(P) - 2\gamma c_e$ and
$w_{e'}(P) = I_{e'}(P) - 2\gamma c_{e'}$ are close, the confusion
between $e$ and $e'$ from samples frequently occurs, leading to high
error probability, and we can show the existence of the infimum $Q$
satisfying $w_e(Q) = w_{e'}(Q)$ as by slightly adjusting the true
distribution $P$. Moreover, we remark that the decaying rate
$J_{e,e'}$ (and thus $K(\gamma)$) is characterized by a trade-off
parameter $\gamma$. The error rate becomes smaller (\ie, higher error
probability) when $\gamma$ nearly meets the condition
$w_e(P) = w_{e'}(P)$, and the weights becomes deterministic with
respect to the samples as $\gamma$ increases since the portion of the
cost in weights grows, resulting to $J_{e,e'} = \infty$ in
\eqref{eq:crate-async}. These interpretations are well-matched to our
numerical results in Section~\ref{sec:simul}.

\smallskip
\noindent{\bf \em Proof sketch.}
The proof of Theorem~\ref{thm:eexponent-async} is presented in
Appendix~\ref{sec:proof-async-alg}, and we describe the proof sketch
for readers' convenience. Our proof largely follows that of the
related work in \cite{tan2011tree} that analyzes an error exponent of
a standard tree structure learning (\ie, known as Chow-Liu algorithm
\cite{chow1968tree}), whose goal is to solely estimate the true data
distribution with no consideration of communication cost. Simply, the
proof idea follows LDP in the following way. The error event
$\set{A}_n({\bm x}^{1:n})$ is expressed as a union of small events
that {\bf ASYNC-ALGO} estimates only one wrong edge (see the
definition of the crossover event in \eqref{eq:cevent-async}), two
wrong edges, and three, etc.  Following LDP, the decaying rate of the
error probability equals to the decaying rate of the most probable
crossover event, which corresponds to the case of only one wrong
edge. In more detail, two minimums in \eqref{eq:eexponent-async}
specify the most-probably error, whose edge set differs from the
optimal data tree structure $\estar(\infty)$ exactly in one edge, ,
\ie, $\estar(\infty) \setminus \{e\} \cup \{e'\}$, where it contains
the non-neighbor node pair $e'$ (as selected in the first
minimization) instead of the most probable replacement edge $e$ in the
unique path along $\estar(\infty)$ (as in the second minimization). To
obtain the minimum crossover rate $J_{e,e'}$, we apply the Sanov's
theorem \cite{bucklew90ldp}, which provides an expression of the
probabilistic relationship between $\hat{P}$ and $P$ via their KL
divergence.  Finally, in addition to the asymptotic decaying rate of
the estimation error probability, we also establish its upper bound of
the error probability in terms of the number $n$ of data samples,
where the first term $(d-1)^2(d-2)/2$ of the bound in
\eqref{eq:sbound-async} implies the number of possible crossover
events, and the second term $\binom{n-1+|\set{X}|^4}{|\set{X}|^4-1}$
represents the number of possible empirical distributions
$\hat{P}_{e,e'}.$

\subsection{Estimation Error of SYNC-ALGO} \label{sec:sync-limit}

We conduct a similar analysis here for {\bf SYNC-ALGO} to what we did
for {\bf ASYNC-ALGO}, which has more complicated issues for the
following reasons: We first denote by $w_e(\hat{P},T)$ in
\eqref{eq:weight-sync} the assigned weight for edge $e$ to stress its
dependence on the corresponding resulting tree structure $T$ and its
associated empirical distribution $\hat{P}$. Then, we need to
investigate the most probable pattern in the rare event through a
certain tree $T$ at some iteration.  Simply, the crossover event for
two edges $e$ and $e'$ occurs if the order of edge weights from the
given finite number of samples becomes reversed to the order of
weights from the true data distribution. Among all possible crossover
events, we are interested in the crossover event under every tree
structure that is obtained on the way of constructing the ideal data
structure, denoted by $\hat{E}^\text{S}(\infty)$. Let $\etrue$ and
$\Ttrue$ be the selected edge and constructed tree at $t$-th iteration
obtained by running {\bf SYNC-ALGO} w.r.t. the true data distribution
$P$, which would finally find $\hat{E}^\text{S}(\infty)$. Then, it is
obvious that $\etrue$ has the unique highest edge weight for $P$,
and the crossover event of our interest is defined as:
\begin{align} \label{eq:cevent-sync} C_{n}(\etrue,e'; \Ttrue)
  \defeq \left\{ w_{\etrue}(\hat{P}; \Ttrue) \leq
    w_{e'}(\hat{P}; \Ttrue) \right\}.
\end{align}
We now state Theorem~\ref{thm:eexponent-sync} that establishes the
decaying rate of the estimation error probability as the number of
data samples grows.

\begin{theorem}[Decaying rate of {\bf
    SYNC-ALGO}] \label{thm:eexponent-sync} For any fixed parameter
  $\gamma \geq 0$,
  \begin{align} \label{eq:ldp-sync}
    \lim_{n \rightarrow \infty} - \frac{1}{n} \log \mathbb{P}(\set{A}_n({\bm x}^{1:n})) \geq K(\gamma),
  \end{align}
  where
  \begin{align} \label{eq:eexponent-sync} K(\gamma) := \min_{t \in
      \{1,\cdots,|V|-1\}} ~ \min_{e' \notin \Tttrue } ~ J_{\etrue,
      e'}(\Ttrue),
  \end{align}
  where $\etrue$ and $\Ttrue$ are the selected edge and
  constructed tree at $t$-th iteration by running {\bf SYNC-ALGO}
  w.r.t. the true data distribution $P$, \ie,
  $w_{\etrue}(P)$ has the maximum edge weight under the tree
  $\Ttrue$, and it is given by: 
  under some tree $T$, for any $e,e'$,
  \begin{eqnarray} \label{eq:crate-sync} J_{e,e'}(T) = \left.
    \begin{cases}
      \underset{Q \in \set{P}(\set{X}^4)}{\inf} &
      \Big\{D_{\text{KL}}(Q \parallel P_{e,e'}) : w_e(Q) = w_{e'}(Q)
      \Big\}, \cr & \text{if} ~ \left\{Q \in \set{P}(\set{X}^4):
        w_e(Q) = w_{e'}(Q) \right\} \neq \emptyset, \cr \quad \infty,
      & \text{otherwise}.
    \end{cases}
                \right.
  \end{eqnarray}
  Moreover, we have the following (finite-sample) upper-bound on the
  error probability: for all $n = 1, 2, \ldots, $
  \begin{align} \label{eq:sbound-sync} \bprob{ \set{A}_n({\bm
        x}^{1:n}) } \leq \frac{(d-1)d(d+1)}{6}\binom{n-1+
      |\mathcal{X}|^{4}}{|\mathcal{X}|^{4}-1} \exp(-n \cdot
    K(\gamma)).
  \end{align}  
\end{theorem}

In Theorem~\ref{thm:eexponent-sync}, as seen in \eqref{eq:ldp-sync},
the error rate function $K(\gamma)$ in \eqref{eq:eexponent-sync}
indeed provides a lower-bound of the actual decaying rate of the error
event $\set{A}_n({\bm x}^{1:n})$, since the crossover event
$C_n(\etrue, e';\Ttrue)$ which estimates an edge
$e' \notin T_{\text{true}}^{t+1}$ rather than $\etrue$ at any $t$-th
iteration does not guarantee that $e'$ is a wrong edge. Intuitively,
the edge weights of {\bf SYNC-ALGO} dynamically change according to a
diameter of $\Ttrue$ as iteration $t$ proceeds, which makes the
characterization of the exact error rate of {\bf SYNC-ALGO} be
non-trivial.

\smallskip
\noindent{\bf \em Proof sketch.}
Due to space limitation, we present the complete proof in
Appendix~\ref{sec:proof-sync-alg}, and we provide a brief proof
sketch. The basic idea is similar to the proof of
Theorem~\ref{thm:eexponent-async}. As mentioned there, the crossover
event $C_n(\etrue, e'; \Ttrue)$ is not a subset of the error event
$\set{A}_n({\bm x}^{1:n})$, and as a result, we provide a lower-bound
of the decaying error rate in the proof, as established by two
minimizations in \eqref{eq:crate-sync}. In particular, the first
minimization is taken over all iterations ($1 \leq t \leq |V|-1$) so
that it selects the iteration where the error occurs in the most
probable way, and the second minimization specifies the non-neighbor
node pair $e'$, which can be estimated instead of $\etrue$, having the
minimum $J_{\etrue, e'}(\Ttrue)$, In other words, the most probable
pattern in the error event of {\bf SYNC-ALGO} is to estimate
$\Ttrue \setminus \{\etrue\} \cup \{e'\}$ attained in two
minimizations in \eqref{eq:eexponent-sync}. For the crossover rate
$J_{\etrue, e'}(\Ttrue)$ in \eqref{eq:crate-sync}, when two edges
$\etrue$ and $e'$ can be clearly differentiated via their edge
weights, since the difference of the cost between two edges dominantly
determines the order of the edge weights, \ie, the condition in
\eqref{eq:crate-sync} does not hold, the crossover event does not
happen, \ie, $J_{\etrue, e'}(\Ttrue) = \infty$. This mostly
corresponds to the situation of a large value of the trade-off
parameter $\gamma$, where the communication cost plays an important
role of the error event, which do not depend on the number of samples
$n$. Otherwise, the crossover rate is attained in a similar way to
\eqref{eq:crate-async}. Finally, we establish the upper bound of the
error probability in terms of the sample size $n$, where the first
term $(d-1)d(d+1)/6$ of the bound in \eqref{eq:sbound-sync}
corresponds to the number of possible crossover events throughout the
entire iterations, and the second term implies the number of possible
empirical distributions $\hat{P}_{\etrue,e'}$.



\section{Numerical Results} \label{sec:simul}

In this section, we provide a set of numerical experiments to validate
our analytical results of {\bf ASYNC-ALGO} and {\bf SYNC-ALGO} under
various numbers of data samples, communication costs, and trade-off
parameters.


\subsection{Setup} \label{sec:setup}

\smallskip
\noindent {\bf \em Physical graph.}
We use a physical network $G=(V,E_P)$ consisting of $20$ nodes forming
a line topology, where node $i$ can directly communicate only with
nodes $i-1$ and $i+1$, see Figure~\ref{fig:phy_line}. We assign some
constant cost of single message-passing for each edge $e=(i,i+1)$:
$c_{i,i+1} = \kappa \times 1.1^i$, except for
$c_{1,2} = 4 \kappa, c_{3,4} = 2\kappa, c_{6,7}= 0.1 \kappa$, where we
appropriately choose $\kappa$ to adjust the scale of total
communication cost of two learning algorithms in the same range, for
clear comparison with the same values of $\gamma$. In the
message-passing between non-neighboring (w.r.t. the physical graph)
node pairs $(i,j)$, we simply assume that it expenses the sum of the
costs when it is passed along the unique shortest multi-hop path
$\Psi((i,j); G)$ on $G$, \ie,
$c_{i,j} = \sum_{e' \in \Psi((i,j); G)} c_{e'}$. For example,
$c_{1,4} = c_{1,2}+c_{2,3}+c_{3,4}$. We use this line topology for an
exemplar physical graph to clearly observe the difference between {\bf
  ASYNC-ALGO} and {\bf SYNC-ALGO}, where it leads to a significantly
huge amount of communication cost for {\bf SYNC-MAP}, due to large
diameter value $\diam(G) = 19$.

\begin{figure}[t!]
  \centering
  \subfigure[Physical graph of $20$ nodes forming a line topology.]
  {
\includegraphics[width=0.33\columnwidth]{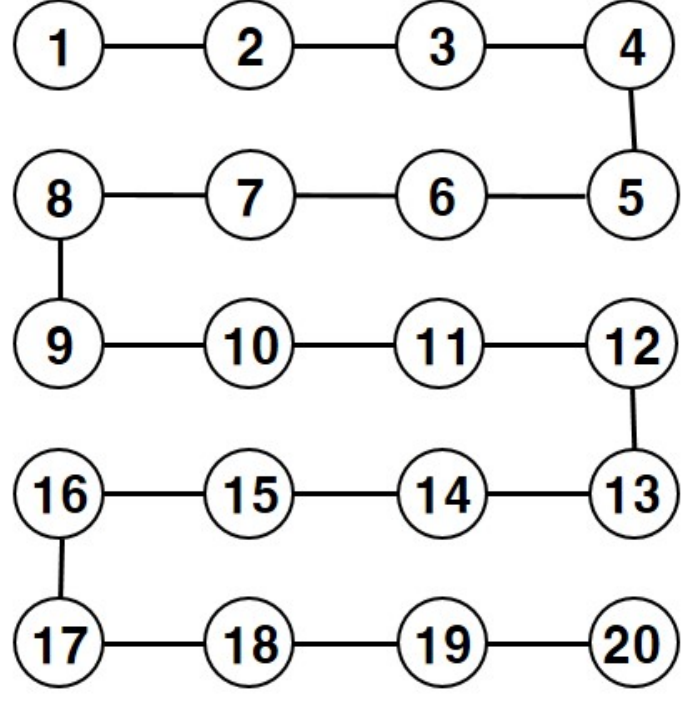}    
    \label{fig:phy_line}
  } \hspace{0.4cm}  
  \subfigure[Data graph of $20$ nodes forming a $3$-regular tree (except for the leaves).]
  {
    \includegraphics[width=0.53\columnwidth]{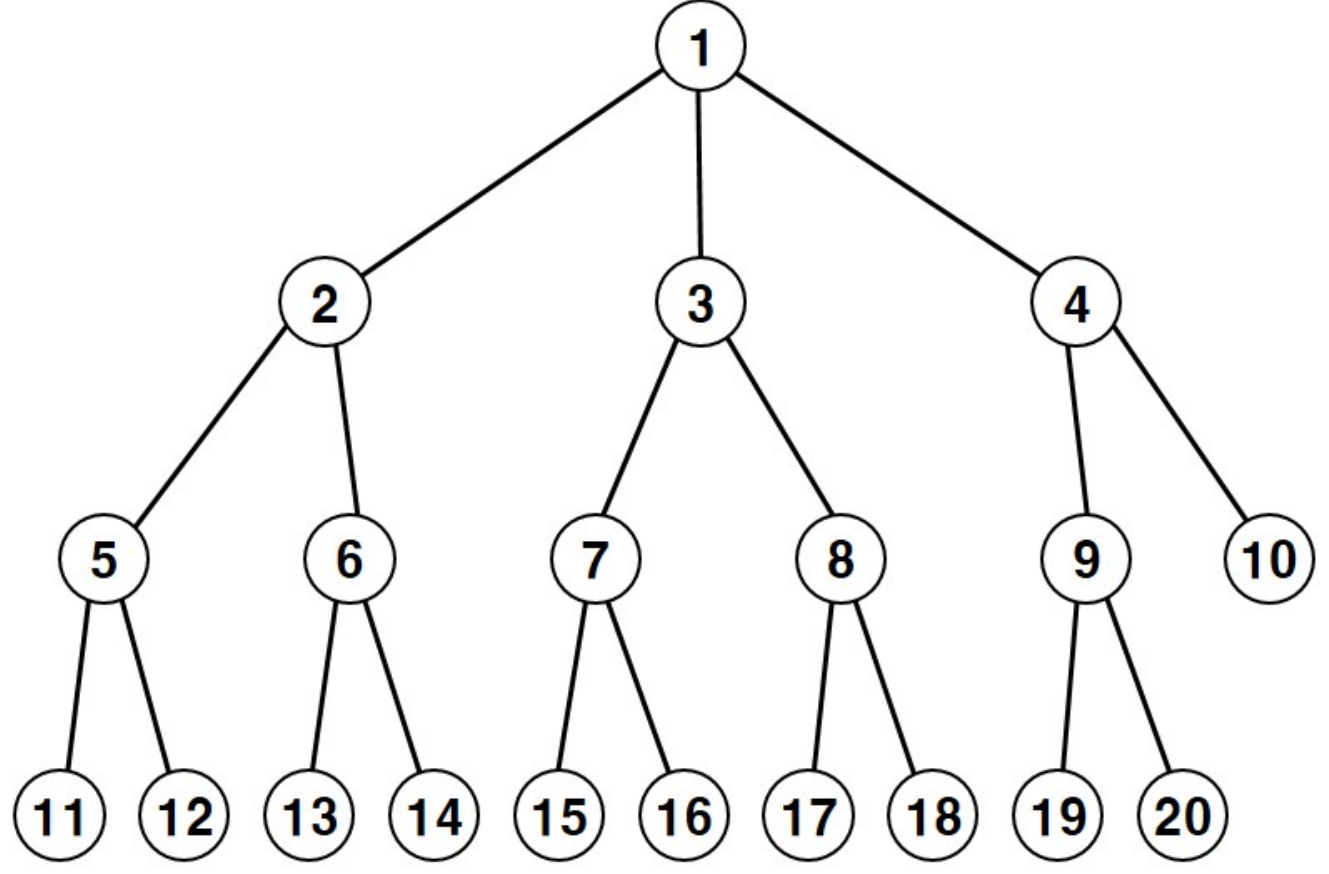}
    \label{fig:data_graph}
  }
\caption{Physical and data graphs used for our simulations.}
\label{fig:data_physical_graph}
\end{figure}

\smallskip
\noindent {\bf \em Data graph.}
As an underlying statistical dependencies among $20$ nodes in the data
graph, we consider a $3$-regular tree $T=(V,E_D)$, except for boundary
nodes, where the node $1$ is a root node and every node has a degree
of $3$ or less, as depicted in Figure~\ref{fig:data_graph}. Each
random variable $X_i$ associated to a node $i$ is set to follow a
Bernoulli distribution. For a root node $1$, it has $P(X_1=0) = 0.7$
and $P(X_1=1) = 0.3$, and for other neighboring node pairs $i$ and
$j$, we set the conditional distribution between $X_i$ and $X_j$ by
\begin{align} \label{eq:exp-data-dist}
  P(X_i = 0 | X_j = 0) = 0.7, \quad \text{and} \quad P(X_i = 0 | X_j = 1) = 0.3
\end{align}
whenever $i < j$. With this setting of per-node distribution, it turns
out that neighboring node pairs have high correlations, and thus have
distinct values of the mutual information.

\smallskip Under this choice of physical and data graphs, we obtain
numerical examples to show the performance of {\bf ASYNC-ALGO} and
{\bf SYNC-ALGO} for various values of trade-off parameter $\gamma,$
ranging from $0$ to $4$, and a fixed $\beta = 1$ in our results. For a
fixed $n \in \mathbb{N}$, we first generate $n$ i.i.d. samples
${\bm x}^{1:n}$ from $P({\bm x})$ in \eqref{eq:exp-data-dist}. Then,
we compute the empirical distribution $\hat{P}({\bm x}^{1:n})$ and the
empirical mutual information of all possible node pairs
$\{ I_e(\hat{P})\}_{e \in V \times V}$. Then, we learn the
cost-efficient data tree by running {\bf ASYNC-ALGO} or {\bf
  SYNC-ALGO}, and estimate how well the proposed algorithms recover
the ideal data graph by investigating the estimation error probability
as $n$ grows.


\begin{figure*}[t!]
  \centering
  \subfigure[Estimated tree structure of $\gamma=0$]
  {
    \includegraphics[width=0.46\columnwidth]{reg_3_2.pdf}
    \label{fig:async_gamma_0}
  } \hspace{0.25cm}  
  \subfigure[Estimated tree structure of $\gamma=0.5$]
  {
    \includegraphics[width=0.4\columnwidth]{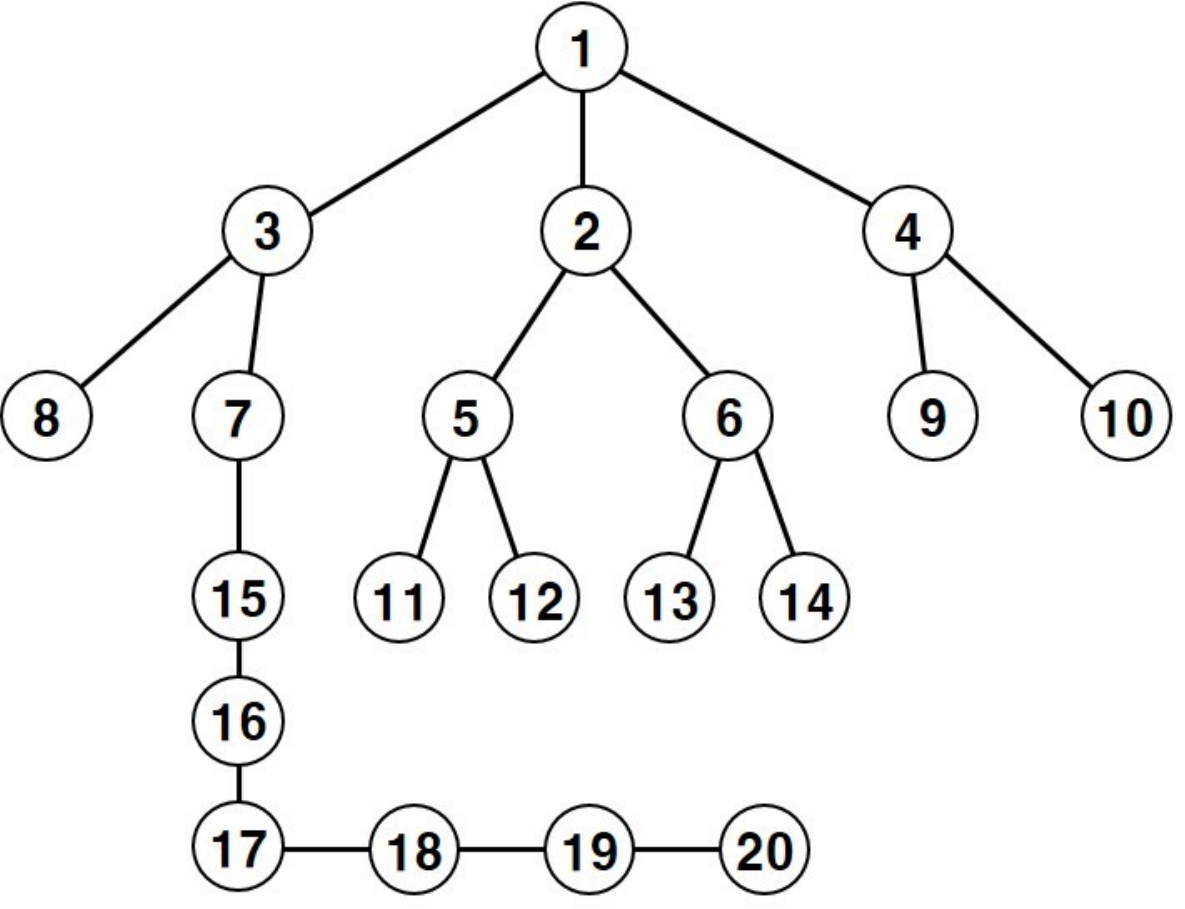}
\label{fig:async_gamma_1}
  }
  \hspace{0.35cm}
  \subfigure[Estimated tree structure of $\gamma=4$]
  {
    \includegraphics[width=0.3\columnwidth]{line2.pdf}
    \label{fig:async_gamma_4}
  } \hspace{0.7cm}
  \subfigure[Trade-off between MAP accuracy and cost] {
    \includegraphics[width=0.47\columnwidth]{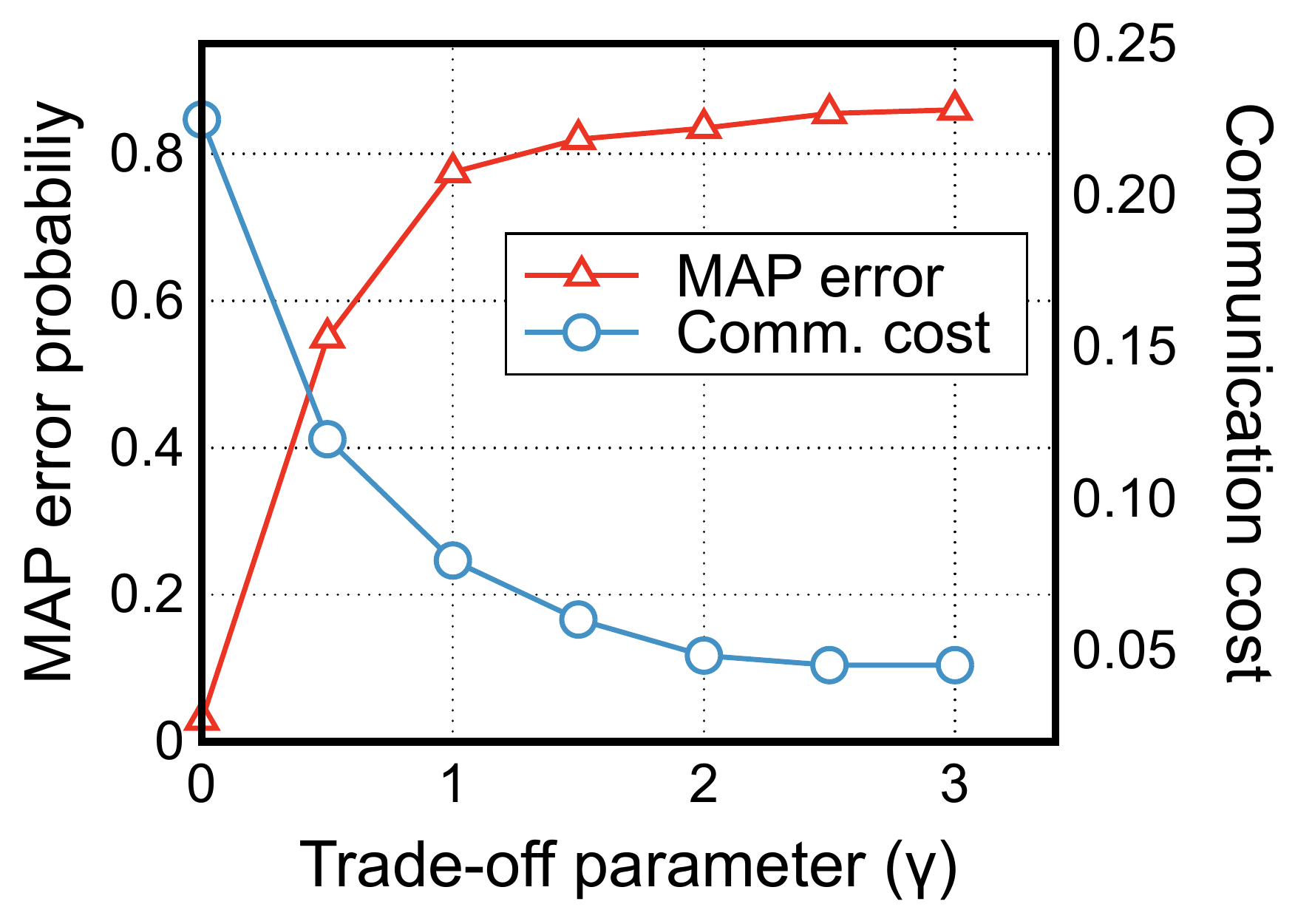}
    \label{fig:async-trade-off}
  } \hspace{-0.35cm}
  \caption{An instance of estimated tree structure by {\bf ASYNC-ALGO}
    with distinct trade-off parameter $\gamma=0,0.5,4$, and the
    trade-off between MAP accuracy and communication cost.}
  \label{fig:ASYNC_struct}
\end{figure*}

\begin{figure*}[t!]
  \centering
  \hspace{-0.1cm}
  \subfigure[Estimated tree structure of $\gamma=0$]
  {
    \includegraphics[width=0.45\columnwidth]{reg_3_2.pdf}
    \label{fig:sync_gamma_0}
  } 
  \subfigure[Estimated tree structure of $\gamma=1$]
  {
    \includegraphics[width=0.49\columnwidth]{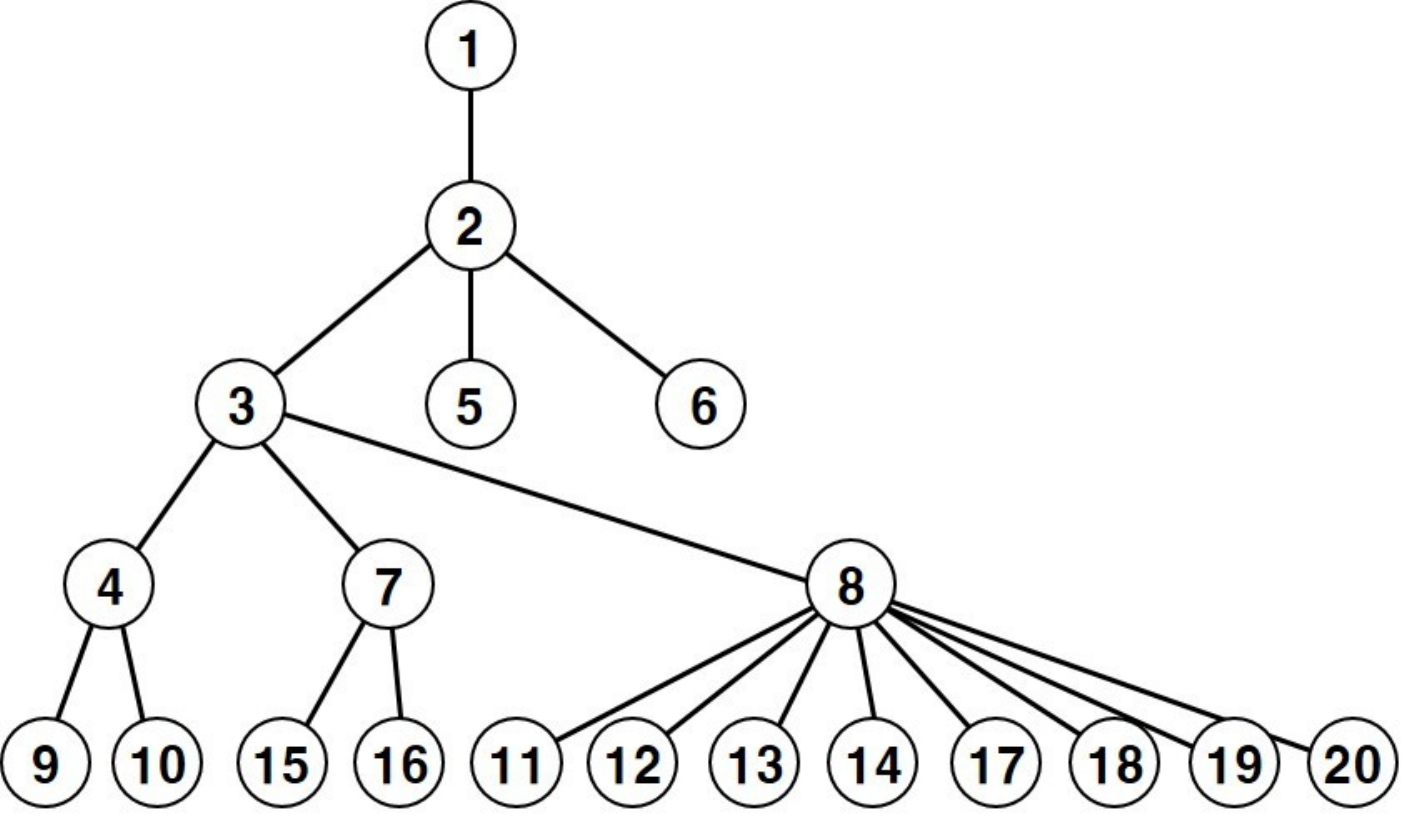}
    \label{fig:sync_gamma_1}
  }
  \hspace{-0.2cm}
  \subfigure[Estimated tree structure of $\gamma=4$]
  {
     \includegraphics[width=0.32\columnwidth]{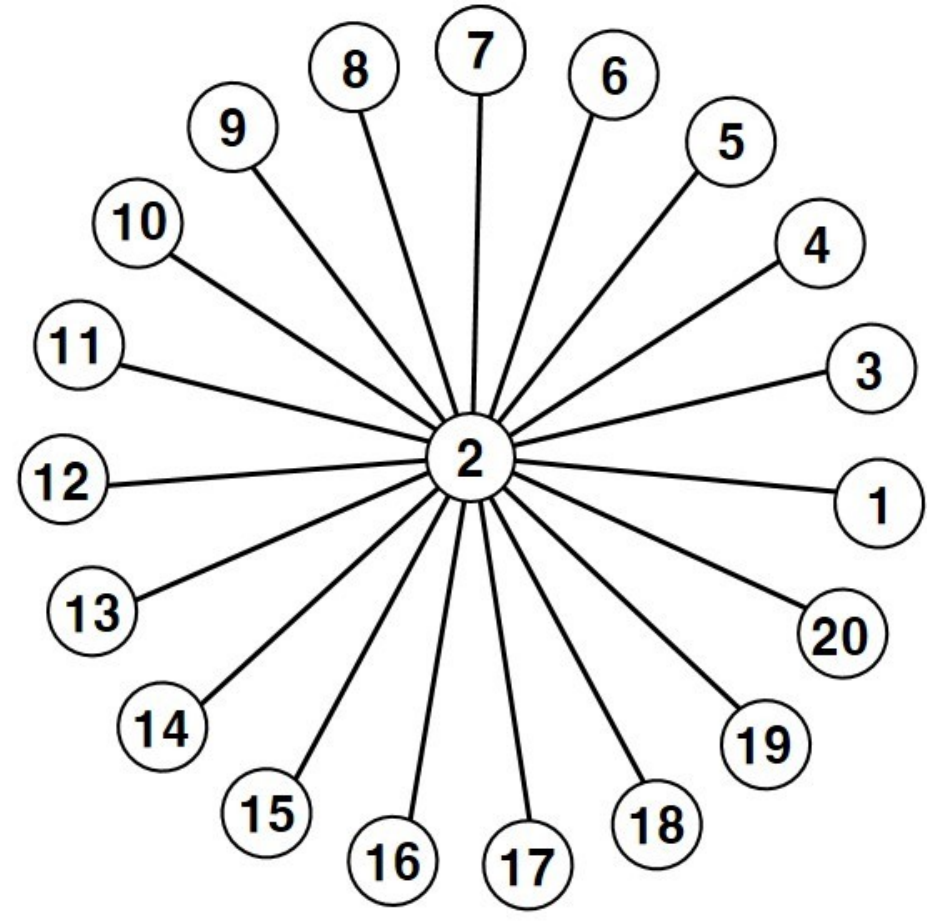}
     \label{fig:sync_gamma_4}
   }
  \hspace{0.3cm}
  \subfigure[Trade-off between MAP accuracy and cost]
  {
     \includegraphics[width=0.47\columnwidth]{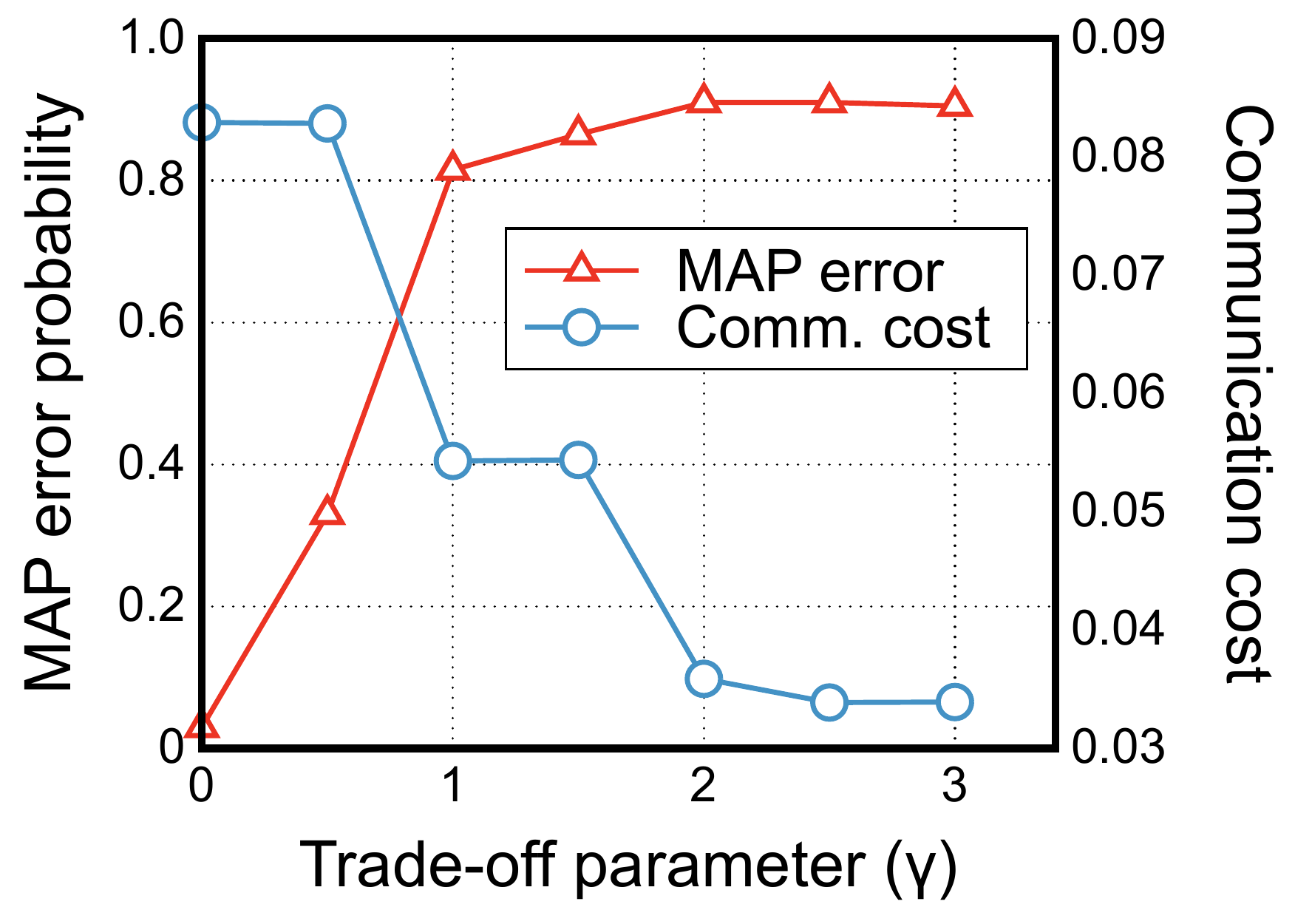}
     \label{fig:sync-trade-off}
   } \hspace{-0.2cm}
  \caption{An instance of estimated tree structure by {\bf SYNC-ALGO}
    with distinct trade-off parameter $\gamma = 0, 1, 4$, and the
    trade-off between MAP accuracy and communication cost.}
  \label{fig:SYNC_struct}
\end{figure*}

\subsection{Results}

\smallskip
\noindent {\bf \em (i) Estimated trees with varying $\gamma$.}
Figures~\ref{fig:ASYNC_struct} and \ref{fig:SYNC_struct} show that the
estimated data trees by {\bf ASYNC-ALGO} and {\bf SYNC-ALGO} for
various $\gamma$. We recall that the value of $\gamma$ parameterizes
the amount of priority for communication cost compared to the
inference quality, see \eqref{eq:CDG-learn}, where smaller $\gamma$
leads to higher priority to the inference quality. In both algorithms,
we observe that they with $\gamma = 0$ estimate the exact data graph
in Figure~\ref{fig:data_graph}, since the goal is to achieve the
highest inference accuracy. However, as $\gamma$ grows, each of two
algorithms estimates a different structure for data tree, since {\bf
  ASYNC-MAP} and {\bf SYNC-MAP} have different forms of communication
costs. In particular, in {\bf ASYNC-ALGO}, as $\gamma$ grows, we
observe that the algorithm produces the estimated data tree with more
resemblance to the physical graph, and finally it estimates the data
tree that is the same as the physical graph with $\gamma = 4$, see
Figures~\ref{fig:async_gamma_1} and~\ref{fig:async_gamma_4}.  We note
that for a large value of $\gamma$, the goal of {\bf ASYNC-ALGO} is to
find a MWST of minimum total cost, which accords with the physical
graph of line-topology. However, the communication cost of {\bf
  SYCN-ALGO} increases in proportion to the {\em diameter} of the
estimated tree, thus it estimates a tree that is of a {\em star-like}
topology, \ie, a tree with a small diameter as seen in
Figures~\ref{fig:sync_gamma_1} and~\ref{fig:sync_gamma_4}, to
significantly reduce the cost, as $\gamma$ grows.

\begin{figure}[t!]
  \centering
  \hspace{-0.5cm}
  \subfigure[Error probability w.r.t. the sample size $n$]
  {
    \includegraphics[width=0.47\columnwidth]{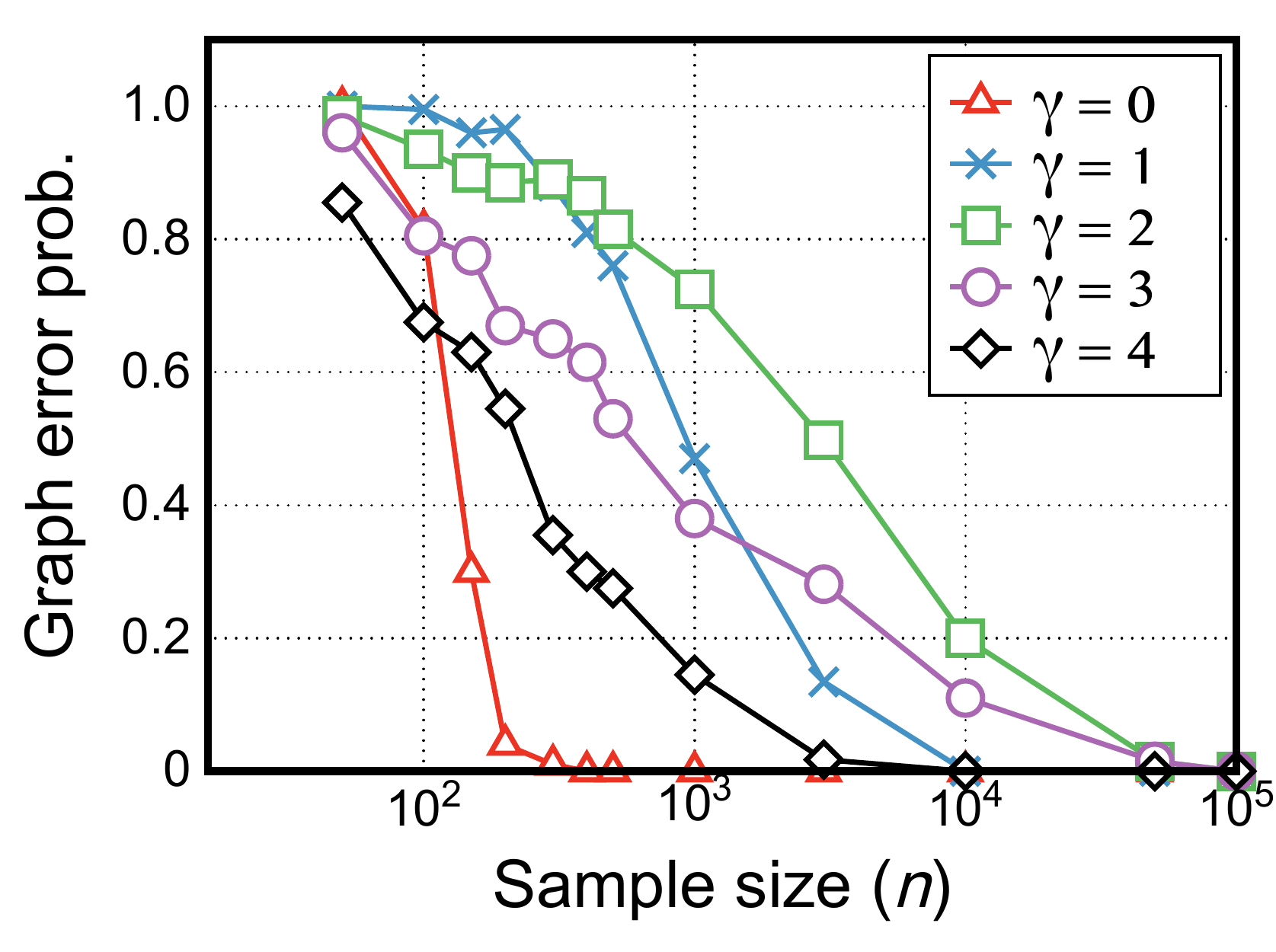}
    \label{fig:async-err-sample}
  } \hspace{-0.05cm}
  \subfigure[Error probability with varying trade-off parameter $\gamma$]
  {
    \includegraphics[width=0.47\columnwidth]{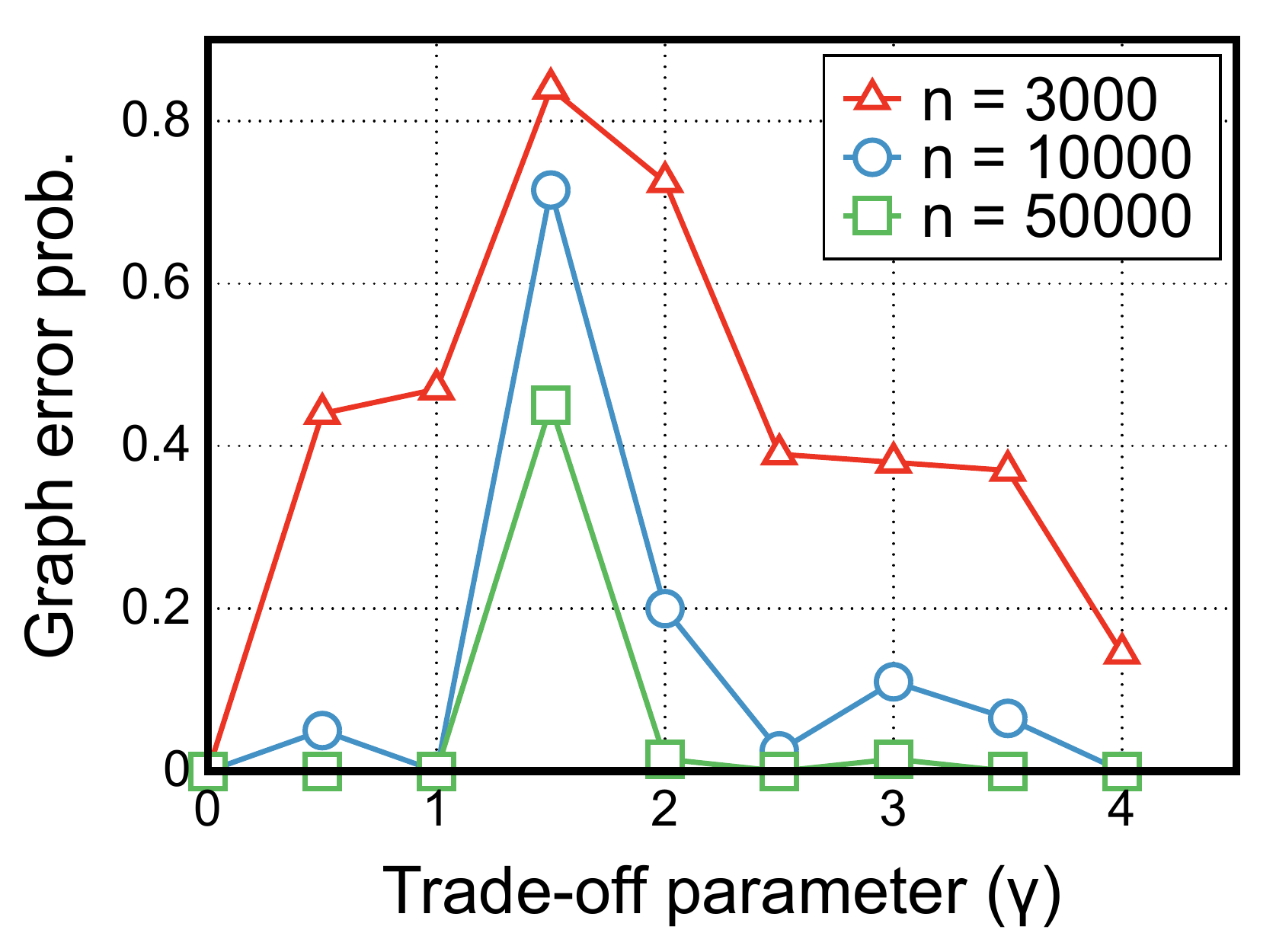}
    \label{fig:async-err-gamma}
  } \hspace{-0.5cm}
  \caption{Error probability of {\bf ASYNC-ALGO}.}
  \label{fig:async}
\end{figure}

\smallskip
\noindent {\bf \em (ii) Quantifying trade-off between
  inference accuracy and cost.}
We now quantify how the trade-off between inference accuracy of the
MAP estimator in \eqref{eq:map} behaves and the total communication
cost is captured for different values of $\gamma$.  To support the
trade-off parameterized by $\gamma$ in the optimization problem in
\eqref{eq:CDG-learn}, we vary $\gamma$ from $0$ to $4$ and plot the
accuracy of MAP estimator and the total cost on the learnt data
dependency graph as the red and blue lines, respectively, in
Figures~\ref{fig:async-trade-off} and ~\ref{fig:sync-trade-off}. In
particular, we run the {\bf ASYNC-ALGO} and {\bf SYNC-ALGO} with
$n = 200$ samples, respectively, and run the max-product algorithm on
the learnt data tree to obtain the MAP estimator. We repeatedly run
for $200$ times, and measure the error probability that the MAP
estimator on the learnt data tree differs from the MAP estimator on
the true data graph, as a metric of inference accuracy. The average
(over the $200$ results) of the communication cost on the learnt data
tree is measured by the form of \eqref{eq:ccost-async} and
\eqref{eq:ccost-sync} for each algorithm. In
Figure~\ref{fig:sync-trade-off}, we observe that the MAP estimation
error and cost with $\gamma = 0.5$ is $0.33$ and $0.083$,
respectively, while those with $\gamma = 2.5$ is $0.91$ and $0.034$,
respectively. The impact of $\gamma$ on the trade-off for two
algorithms seems similar, as seen in Figures~\ref{fig:async-trade-off}
and~\ref{fig:sync-trade-off}.

\begin{figure}[t!]
  \centering
  \hspace{-0.5cm}
  \subfigure[Error probability w.r.t. the sample size $n$]
  {
    \includegraphics[width=0.47\columnwidth]{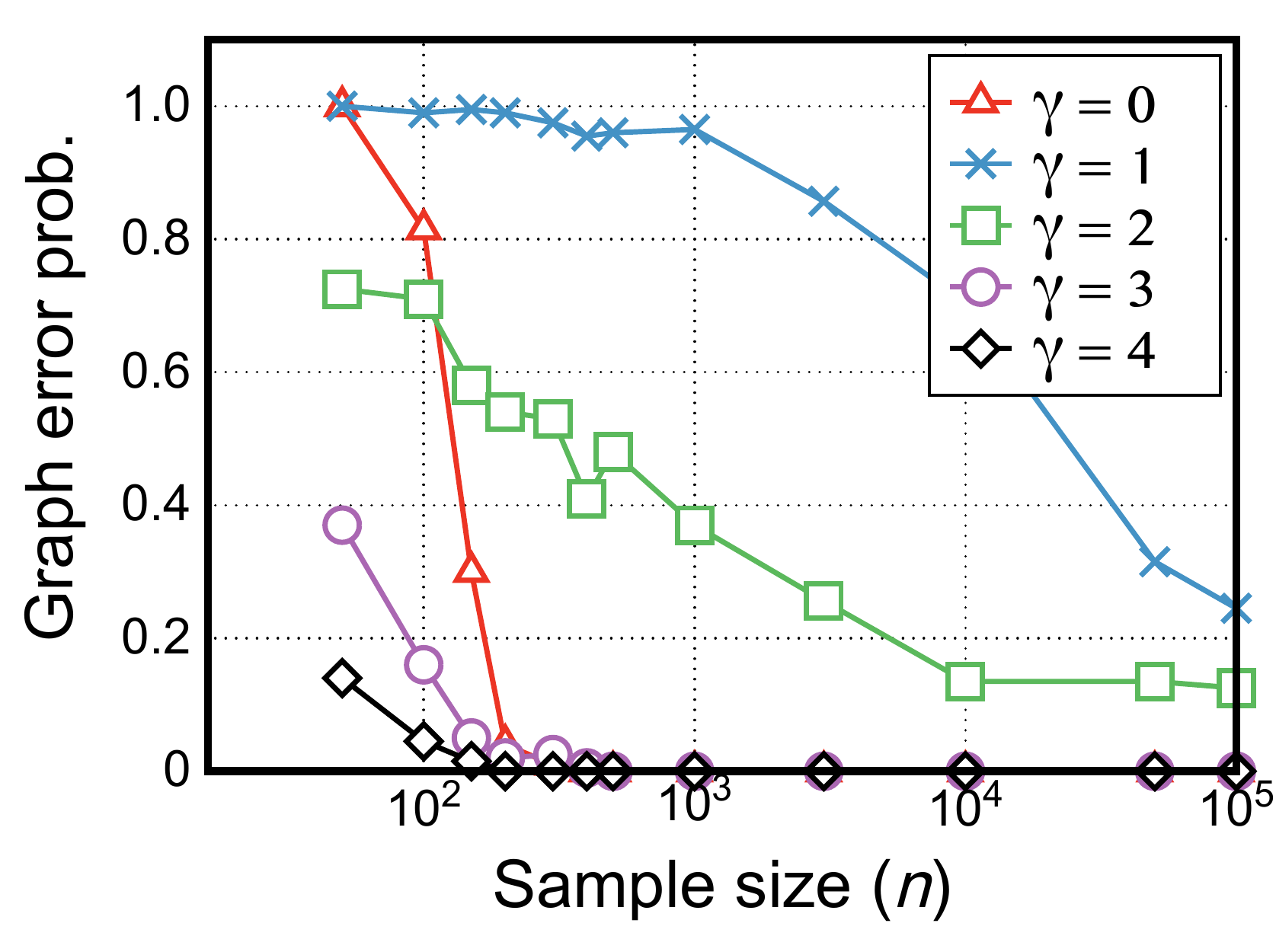}
    \label{fig:sync-err-sample}
  } \hspace{-0.05cm}
  \subfigure[Error probability with varying trade-off parameter $\gamma$]
  {
    \includegraphics[width=0.47\columnwidth]{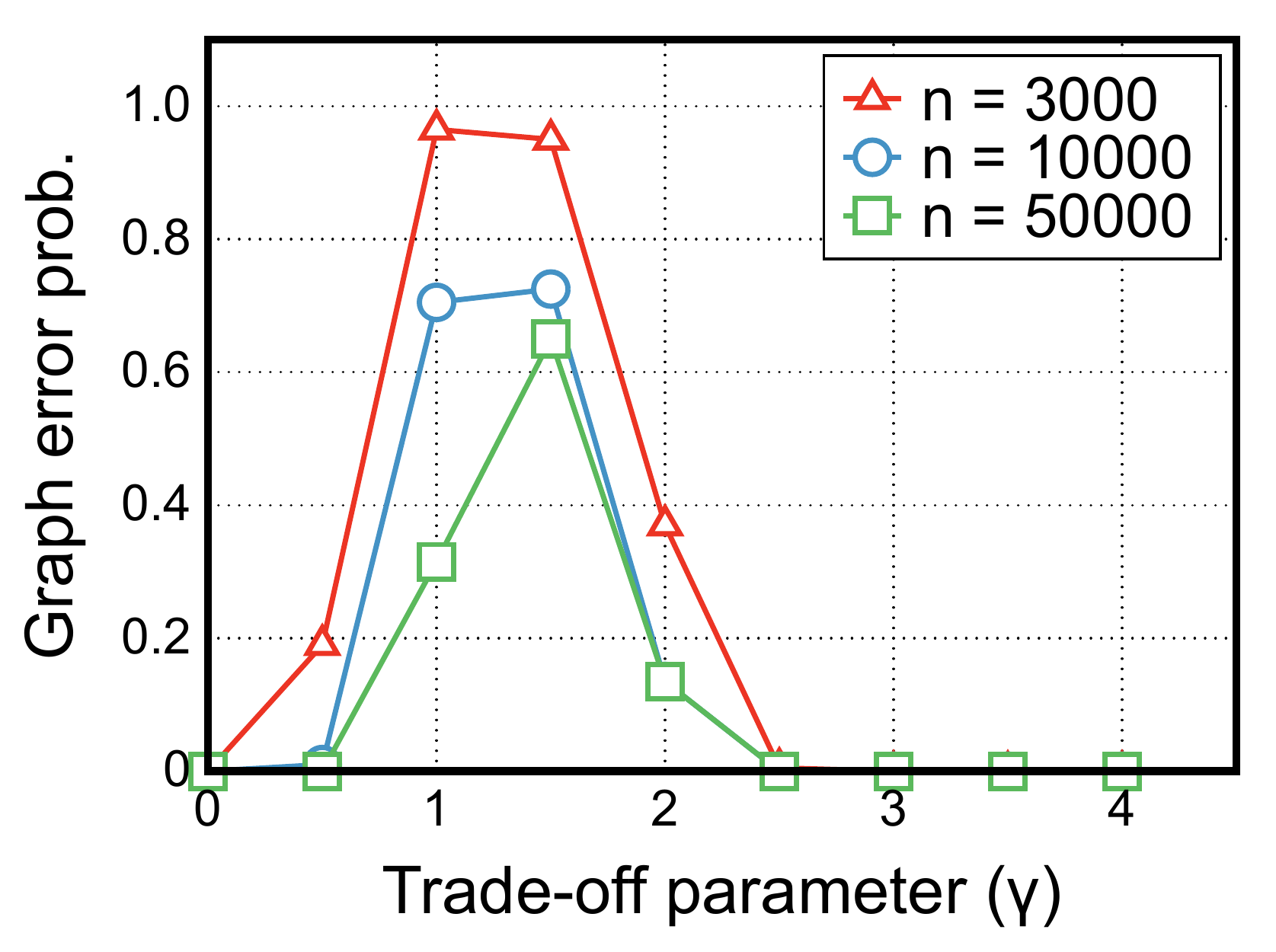}
    \label{fig:sync-err-gamma}
  } \hspace{-0.5cm}
  \caption{Error probability of {\bf SYNC-ALGO}.}
  \label{fig:sync}
\end{figure}

\smallskip
\noindent {\bf \em (iii) Impact of data sample size on graph
  estimation accuracy.}
Finally, we demonstrate the theoretical findings in
Theorems~\ref{thm:eexponent-async} and~\ref{thm:eexponent-sync} on the
decaying rate of the error probability w.r.t. the number of samples
$n$ for various values of $\gamma$. In both {\bf ASYNC-ALGO} and {\bf
  SYNC-ALGO}, for a fixed $\gamma$, we run both algorithms for $200$
times each, and measure their error probabilities. In
Figures~\ref{fig:async-err-sample} and~\ref{fig:sync-err-sample}, we
observe that the error probability $\mathbb{P}(\set{A}_n)$ for every
$\gamma$ decays exponentially as the sample size $n$ increases, as
established in \eqref{eq:sbound-async} and \eqref{eq:sbound-sync}. It
is interesting to see that a different choice of $\gamma$ leads to a
different decaying rate, which can be understood by our analytical
findings of the crossover rate in \eqref{eq:crate-async} and
\eqref{eq:crate-sync}, simply given by:
\begin{eqnarray*} J_{e,e'}(T) = \underset{Q \in
    \set{P}(\set{X}^4)}{\inf} & \Big\{D_{\text{KL}}(Q \parallel
  P_{e,e'}) : w_e(Q) = w_{e'}(Q) \Big\},
\end{eqnarray*}
where the edge weights for {\bf ASYNC-ALGO} and {\bf SYNC-ALGO} are
assigned in different forms, yet depending on the value of $\gamma$,
as seen in \eqref{eq:weight-async} and \eqref{eq:weight-sync}. Some
choice of $\gamma$ makes a difference of the corresponding edge
weights highly small, so that it becomes easier to estimate wrong
edges with an insufficient number of samples. In our simulation, {\bf
  ASYNC-ALGO} with $\gamma = 2$ shows higher error probability of
$0.2$ with $n = 10^{4}$ samples, while that with $\gamma = 0$ achieves
almost $0$ error probability with less than $3000$ samples, see
Figure~\ref{fig:async-err-sample}. This impact of $\gamma$ on the
error probability is presented in Figures~\ref{fig:async-err-gamma}
and~\ref{fig:sync-err-gamma} for both algorithms, where for large
$\gamma$, we observe the error probability decays at a higher rate in
general, since the priority to the inference accuracy is
insignificant, leading to less chance of experiencing the crossover
event.

\section{Conclusion} \label{sec:conclusion}

In many multi-agent networked systems, a variety of applications
involve distributed in-network statistical inference tasks, such as
MAP (maximum a posteriori), exploiting a given knowledge of
statistical dependencies among agents. When agents are
spatially-separated, running an inference algorithm leads to a
non-negligible amount of communication cost due to inevitable
message-passing, coming from the difference between data dependency
and physical connectivity. In this paper, we consider a structure
learning problem which recovers the statistical dependency from a set
of data samples, which also considers the communication cost incurred
by the applied distributed inference algorithms to the learnt data
graph. To this end, we first formulate an optimization problem
formalizing the trade-off between inference accuracy and cost, whose
solution chooses a tunable point in-between them. As an inference
task, we studied the distributed MAP and their two implementations
{\bf ASYNC-MAP} and {\bf SYNC-MAP} that have different cost generation
structures. In {\bf ASYNC-MAP}, we developed a polynomial time,
optimal algorithm, inspired by the problem of finding a maximum weight
spanning tree, while we proved that the optimal learning in {\bf
  SYNC-MAP} is NP-hard, thus proposed a greedy heuristic. For both
algorithms, we then established how the error probability that the
learnt data graph differs from the ideal one decays as the number of
samples grows, using the large deviation principle.


\begin{small}
\bibliographystyle{unsrt}
\bibliography{ref}
\end{small}


 \appendix
 \section{Appendix} \label{sec:appendix}


\subsection{Proof of Theorem~\ref{thm:sync-np}} \label{sec:proof-np}

To prove the NP-hardness of the problem {\bf SYNC}$({\bm n})$ in
\eqref{eq:CDG-learn-sync}, we need some assumptions. Per-message
communication cost satisfies following: {\em (a)} for all $i,j$, we
have $c_{i,j} > 0$ and {\em (b)} for all distinct $i,j,k$, we have
$c_{i,j} + c_{j,k} \geq c_{i,k}$. We assume that mutual information
and communication cost are defined on the complete graph
$\set{G}=(\set{V},\set{E})$. As we expressed in \eqref{eq:obj-diam},
for given ${\bm x}^{1:n}$, the objective function of a tree
$T_Q \in \set{T}$ for a fixed $\gamma \in \mathbb{R}_{\geq 0}$ is
\begin{align*} {\set{O}}^\gamma(T_Q, \hat{P}({\bm x}^{1:n})) = \sum_{e
    \in \set{E}_Q} I(\hat{P}_e) - 2\gamma D_{T_Q} \cdot c_e.
\end{align*}
For convenience, we denote by $c_e^\gamma = 2\gamma c_e$ in the
remaining of the paper. Then, the problem {\bf SYNC}$({\bm n})$ is
equivalent to find a tree
$T^\gamma({\bm x}^{1:n}) = \arg\max_{T_Q \in \set{T}}~
\set{O}^\gamma(T_Q).$


\begin{proof}
  We reduce the problem {\bf SYNC}$({\bm n})$ in
  \eqref{eq:CDG-learn-sync} to the known NP-complete problem {\em
    Exact Cover by 3-sets}, simply called {\bf X3C} problem. We first
  describe what the {\bf X3C} problem is.

  \noindent {\bf \em Exact Cover by 3-sets (X3C).}
  Given a set $S = \{1,2,\cdots,3s\}$ of nodes and a set
  $F = \{f_1,f_2,\cdots,f_q\}$ of $3$-element subsets of $S$, {\bf
    X3C} problem is the {\em decision problem} which determines that
  whether there exists $F' \subset F$ such that (i) the union of the
  elements of $F'$ is $S$ and (ii) the intersection of any two
  elements of $F'$ is an empty set. It is known to be NP-complete.

  To prove the NP-hardness of {\bf SYNC}$({\bm n})$ in
  \eqref{eq:CDG-learn-sync}, we build a specific (tree) data
  distribution $\bar{P}(x)$ with
  $T_{\bar{P}} = (\set{V}, \set{E}_{\bar{P}})$ for the $d = 3s+q+3$
  node variables in $\set{V}$ and a specific cost functions
  $\bar{c}^\gamma \defeq \{\bar{c}_{i,j}^\gamma\}_{(i,j) \in \set{V}
    \times \set{V}}$, so that the corresponding optimization problem
  can be converted to the {\bf X3C} problem. We prove that under such
  a specific situation, the diameter of the optimal tree in
  \eqref{eq:sol-diam} should be $4$ and we can solve the {\bf X3C}
  problem if and only if we have optimal tree solution.

  \noindent {\bf \em Construction of the joint distribution
    $\mathbf{\bar{P}(x})$.} There are $3s+q+3$ node variables in
  $\set{V}$, where $X = \{X_1, \cdots, X_{3s}\}$ denote the element
  nodes and $Y = \{Y_1, \cdots, Y_q\}$ denote the subset nodes. The
  remaining three nodes are denoted by $Z_0,Z_1$ and $Z_2$, \ie,
  $\set{V} = X \cup Y \cup \{Z_0, Z_1, Z_2\}$. First, we construct a
  tree $T_{\bar{P}}$ (whose distribution would be defined below),
  where the node variable $Z_0$ is a root node and the subset nodes
  $Y_1, \cdots, Y_q$ are connected to $Z_0$. Each subset node of $Y$
  is connected to the element nodes of $X$ up to $3$ nodes, and the
  element nodes are leaf nodes in the tree. Finally, $Z_1$ is
  connected to $Z_0$ and $Z_2$ and $Z_2$ has a single connection to
  $Z_1$. It is obvious that $T_{\bar{P}}$ is a spanning tree with set
  of nodes $\set{V}$, and there exist four kinds of edges in
  $\set{E}_{\bar{P}}$: (i) $(X_i,Y_j)$ for some $X_i \in X$ and
  $Y_j \in Y$, (ii) $(Y_j, Z_0)$ for all $Y_j \in Y$, (iii)
  $(Z_0,Z_1)$ and (iv) $(Z_1,Z_2)$, and finally its diameter is $4$,
  \ie, $D_{T_{\bar{P}}} = 4$.
  
  We now construct the tree distribution $\bar{P}(x)$ of the tree
  $T_{\bar{P}}$ as follows. For all node variables $v \in \set{V}$, we
  set the marginal distribution as Bernoulli distribution with
  probability $\frac{1}{2}$. For marginal distribution of the edges,
  we set as follows: for $(X_i,Y_j)$ and
  $(Y_j,Z_0) \in \set{E}_{\bar{P}}$,
  \begin{align} \label{eq:edgejoint} 
    \bar{P}(X_i=0, Y_j=0) &= \bar{P}(X_i=1, Y_j=1) = \frac{1}{4} + \frac{1}{2}\delta_d, \cr
    \bar{P}(X_i=0, Y_j=1) &= \bar{P}(X_i=1, Y_j=0) = \frac{1}{4} - \frac{1}{2}\delta_d, \cr
    \bar{P}(Y_j=0, Z_0=0) &= \bar{P}(Y_j=1, Z_0=1) = \frac{1}{4} + \frac{1}{2}\delta_d, \cr
    \bar{P}(Y_j=0, Z_0=1) &= \bar{P}(Y_j=1, Z_0=0) = \frac{1}{4} - \frac{1}{2}\delta_d,
  \end{align}
  where $\delta_d$ is a positive constant determined by the number of
  nodes $d$. For remaining edges $(Z_0,Z_1)$ and $(Z_1,Z_2)$, we set
  \begin{align*}
    \bar{P}(Z_0=0, Z_1=0) &= \bar{P}(Z_0=1, Z_1=1) = \frac{9}{10}, \cr
    \bar{P}(Z_0=0, Z_1=1) &= \bar{P}(Z_0=1, Z_1=0) = \frac{1}{10}, \cr
    \bar{P}(Z_1=0, Z_2=0) &= \bar{P}(Z_1=1, Z_2=1) = \frac{9}{10}, \cr
    \bar{P}(Z_1=0, Z_2=1) &= \bar{P}(Z_1=1, Z_2=0) = \frac{1}{10}.
  \end{align*}
  The value of $\frac{9}{10}$ can be any arbitrary constant close to
  $1$ and $\delta_d$ should decrease to $0$ as $d \rightarrow \infty$.

  We remark the following Corollary, which is an obvious result from
  the property of the tree distribution in \eqref{eq:tree-dist}.
  \begin{corollary} \label{corr:tree-path} Given a tree distribution
    $P(x)$ in \eqref{eq:tree-dist} with tree structure
    $T_P = (\set{V},\set{E}_P)$, for arbitrary node pair $(a,b)$,
    there is an unique path
    ${\rm W}((a,b); \set{E}_P) \defeq \{ W_0(=a), \cdots, W_l(=b) \}$
    between $X_a$ and $X_b$, such that $(W_k,W_{k+1}) \in \set{E}_P$
    for $1 \leq k \leq l-1$. Moreover, the joint distribution for the
    path ${\rm W}((a,b); \set{E}_P)$ is given by
    \begin{align} \label{eq:tree-path}
      &P(W_{0}
      =w_{0},...,W_{k}=w_{k},...,W_{l}=w_{l}) \cr & \quad = P(W_{0}=w_{0}) \prod
      _{k=0}^{l-1} P(W_{k+1}=w_{k+1} | W_{k}=w_{k}).
    \end{align}
  \end{corollary}

  
  From the Corollary~\ref{corr:tree-path}, the mutual information
  among the node variables $X \cup Y \cup \{Z_0\}$ is determined by
  the length of the path, \ie, the number of hops, between them. Since
  the tree has a dimaeter of $4$, a length of the path between
  arbitrary nodes in $X \cup Y \cup \{Z_0\}$ is either $1,2,3,$ or $4$
  with corresponding mutual information of $I_1, I_2, I_3,$ or $I_4$,
  respectively. For example, between two nodes $X_i$ and $X_k$, there
  is a path $(X_i,X_j,X_k)$ in tree $T_{\bar{P}}$. From the simple
  calculation, the joint distribution is
  \begin{align*}
    \bar{P}(X_i=1,X_k=1) & = \bar{P}(X_i=1|X_j=0)\bar{P}(X_j=0|X_k=1)\bar{P}(X_k=1) \cr
    & \quad + \bar{P}(X_i=1|X_j=1)\bar{P}(X_j=1|X_k=1)\bar{P}(X_k=1) \cr
    & = \frac{1}{2} \left( \left(\frac{1}{2} + \delta_d\right)^2 + \left(\frac{1}{2} - \delta_d\right)^2 \right)
    = \frac{1}{4} + \frac{1}{2} \cdot 2 \delta_d^2.
  \end{align*}
  Consequently, we can easily check that
  \begin{align*} 
    I_k = \left( \frac{1}{2} + m_k\right) \ln (1+2m_i) + \left( \frac{1}{2} - m_k\right) \ln (1-2m_k),
  \end{align*}
  where $m_k = 2^{k-1} \delta_d^k$ for $k=1,2,3,4$. By using the
  Taylor approximation for $\ln$ function, we have as
  $\delta_d \rightarrow 0$,
  \begin{align*}
    I_k \cong 4m_k^2 \cong (2 \delta_d)^{2k}.
  \end{align*}
  Moreover, mutual information among the variables $Z_i$s are given by
  \begin{align*}
    I(Z_0,Z_1) = I(Z_1,Z_2) = \alpha_1 = 0.368, \quad I(Z_0,Z_2) = \alpha_2 = 0.237.
  \end{align*}
  We also can get the exact mutual information value between the node
  variables in $X \cup Y \cup \{Z_0\}$ and $Z_1$ or $Z_2$, however, it
  is obvious that it is less than $I_1$, which is enough for our
  remaining proof.

  \smallskip
  \noindent {\bf \em Construction of the communication cost functions
    $\mathbf{\bar{c}^\gamma}$.} Before the construction of the cost
  functions $\{ \bar{c}_{i,j}^\gamma \}_{(i,j) \in \set{E}}$, we
  consider a supergraph $\set{S} = (\set{V}, \set{E}_{\set{S}})$ of
  the constructed tree $T_{\bar{P}}$. The supergraph $\set{S}$ has the
  same node set $\set{V}$ but super edge set $\set{E}_{\set{S}}$ where
  all subset nodes in $Y$ are connected to the {\em exactly} $3$
  element nodes in $X$. 

  Now, we can classify the edges in the complete graph
  $\set{G} = (\set{V}, \set{E})$ into the following $9$ types
  according to its mutual information and communication cost values.
  \begin{align*}
    &\textsf{T1}: (X_i,Y_j) \in \set{E}_{\bar{P}}, ~ \textsf{T2}: (X_i,Y_j) \in \set{E}_{\set{S}} \setminus \set{E}_{\bar{P}}, ~ \textsf{T3}: (X_i,Y_j) \in \set{E} \setminus \set{E}_{\set{S}}, \cr
    &\textsf{T4}: (Y_i,Y_j) \in \set{E}, \quad \textsf{T5}: (Y_j,Z_0) \in \set{E}, \quad \qquad \textsf{T6}: (X_i,Z_0) \in \set{E}, \cr
    &\textsf{T7}: (X_i,X_j) \in \set{E}, ~\mbox{}~\mbox{}~ \textsf{T8}: (Z_0,Z_1), (Z_1,Z_2), \quad \textsf{T9}: \text{others}.
  \end{align*}
  The corresponding mutual information and cost function of each type
  of edges are presented in Table~\ref{tbl:mi-cost}. The cost of edges
  in \textsf{T8} and $(Z_1,X_i)$ and $(Z_1,Y_j)$ in \textsf{T9} are
  set to be $\kappa$, and the cost of $(Z_2,X_i)$ and $(Z_2,Y_j)$ in
  \textsf{T9} are set to be $2\kappa$, where $I_1 << \kappa$ as
  $s \rightarrow \infty$. One can easily check that the constructed
  cost functions $\{\bar{c}_{i,j}^\gamma\}_{(i,j) \in \set{E}}$
  satisfy the triangle inequality.
  
  \begin{table}[t!]
    \centering
    \caption{Mutual information $I(\bar{P})$ and the communication
      cost $\bar{c}_{i,j}^\gamma$ for each type of eadges}
    \label{tbl:mi-cost}
    \begin{tabular}{|c|c|}
      \hline 
      Edge type  & (Mutual information, Communication cost) \\  
      \hline 
      \textsf{T1}	& ($I_{1},I_{1}$) \\
      \hline 
      \textsf{T2}	& ($I_{3}, \frac{3}{4}I_{1}+ \frac{1}{4}I_{3}$) \\
      \hline 
      \textsf{T3}	& ($I_{3}, \frac{5}{4}I_{1}+ \frac{1}{4}I_{2}+ \frac{1}{4}I_{3}$) \\
      \hline 
      \textsf{T4}	& ($I_{2}, \frac{3}{4}I_{1}+ \frac{1}{4}I_{2}$)  \\
      \hline 
      \textsf{T5}	& ($I_{1}, \frac{3}{2}I_{1}$)  \\
      \hline 
      \textsf{T6}	& ($I_{2}, \frac{11}{8}I_{1}+ \frac{1}{4}I_{2}$)  \\
      \hline 
      \textsf{T7}	& ($I_{2}, \frac{11}{8}I_{1}+ \frac{1}{4}I_{2}$)	\\
      \hline 
      \textsf{T8}	& ($\alpha_{1}=0.368, \kappa$) 	\\
      \hline 
      \textsf{T9}	& (Less than $I_{1}$ except $\alpha_{2}=I(Z_{0},Z_{2})$, $\kappa$ or $2\kappa$)	\\
      \hline 
    \end{tabular}
  \end{table}
  
  \smallskip
  \noindent {\bf \em Optimal spanning tree $T^\star$.} From the
  objective function in \eqref{eq:obj-diam}, our optimal spanning tree
  is a maximum weight spanning tree (MWST) where the weight of edge
  $(i,j)$ is defined as
  $w_{i,j} = I(P_{i,j}) - D_{T^\star}\cdot c_{i,j}^\gamma$. There are
  two main reasons that difficulty of this problem arise. First, the
  MWST with the weight $I(P_{i,j}) - 2 \cdot c_{i,j}^\gamma$ could be
  different from the MWST with the weight
  $I(P_{i,j}) - 6 \cdot c_{i,j}^\gamma$. Second, MWST with the weight
  from a specific diameter value has no guarantee that it has the
  certain diameter. For example, even though we get a MWST with the
  weight $I(P_{i,j}) - 2 \cdot c_{i,j}^\gamma$, we cannot assure its
  diameter is $2$. To handle these issues, we will show that under the
  constructed scenario of $\bar{P}$ and $\bar{c}^\gamma$ in
  Table~\ref{tbl:mi-cost}, the diameter of the optimal spanning tree
  is always 4. For simplicity, we denote by $T^\star$ by the optimal
  spanning tree under the $\bar{P}$ and $\bar{c}^\gamma$, \ie,
  $T^\star$ is a maximum weight spanning tree with the weight
  $I(\bar{P}_{i,j}) - D_{T^\star} \cdot \bar{c}_{i,j}^\gamma$, and
  moreover, it is attained when $D_{T^\star} = 4$.

  \smallskip
  \noindent {\bf \em Reduction to the X3C Problem.} With the
  supergraph $\set{S}$, we can make an instance for {\bf X3C}
  problem. The nodes $X_1, \cdots, X_{3s}$ and $Y_1, \cdots, Y_q$
  correspond to the element nodes and subset nodes in {\bf X3C}
  problem, respectively. The $3$ element nodes connected to each
  subset node $Y_j$ stand for the $3$-element subsets of $Y_j$. It is
  clear that we have total $d = 3s+q+3$ number of nodes. For
  simplicity, we denote by $T_k^* = (\set{V},\set{E}_k^*)$ the optimal
  spanning tree with a fixed diameter $k$. That is, we re-express the
  objective function for a fixed $k = 2, \cdots, d-1$,
  \begin{align*}
    \bar{\set{O}}^\gamma_k(T) = \sum_{(i,j) \in \set{E}_T} I(\bar{P}_{i,j}) - k \cdot \sum_{(i,j) \in \set{E}_T} \bar{c}_{i,j}^\gamma,
  \end{align*}
  and thus
  $T_k^* = \underset{T \in \set{T}}{\arg\max}
  ~\bar{\set{O}}^\gamma_k(T).$
  It is obvious that the optimal solution tree $T^\star$ is the best
  spanning tree among $\{T_k^*\}_{k = 2, \cdots, d-1}$. We now state
  our main result in following Lemma that if we find $T^\star$, we can
  decide whether there is a solution for the {\bf X3C} problem.
  
  \begin{lemma} \label{lem:red-x3c} Suppose that the mutual
    information $I(\bar{P})$ and cost functions $\bar{c}^\gamma$ are
    defined as in Table~\ref{tbl:mi-cost}. Let
    $\delta_d = \frac{\alpha_1 - \alpha_2}{\sqrt{4 (3s + q)}}$ and
    $\kappa = \frac{9}{8}s I_1.$ Then, there is a solution of the
    corresponding {\bf X3C} problem if and only if the diameter of
    $T^\star$ is $4$, \ie, $D_{T^\star} = 4$, and
    \begin{align*}
      \set{O}^\gamma(T^\star) &= \sum_{(i,j) \in \set{E}_{T^\star}} I(\bar{P}_{i,j}) - D_{T^\star} \sum_{(i,j) \in \set{E}_{T^\star}} \bar{c}_{i,j}^\gamma \cr
      &= 2 \alpha_1 - (11s + 3q)I_1 - 8 \kappa.
    \end{align*}
  \end{lemma}

  \begin{proof}
    Suppose there is a solution of the {\bf X3C} problem.

    \begin{table}[t!]
      \centering
      \caption{Edge weight
        $I(\bar{P}_{i,j}) - D_T \cdot \bar{c}_{i,j}^\gamma$ with
        different $D_T$}
      \label{tbl:weight-diam}
      \begin{tabular}{|c|c|c|c|}
        \hline 
        Edge type & Weight for $D_T = 3$ & $D_T = 4$ & $D_T = 5$ \\ 
        \hline 
        \textsf{T1} & $-2I_{1}$ & $-3I_{1}$ & $-4I_{1}$ \\
        \hline 
        \textsf{T2} & $-\frac{9}{4}I_{1} + \frac{1}{4}I_{3}$ & $-3I_{1}$ & $-\frac{15}{4}I_{1} - \frac{1}{4}I_{3}$ \\
        \hline 
        \textsf{T3} & $-\frac{15}{4}I_{1} - \frac{3}{4}I_{2} + \frac{1}{4}I_{3}$ & $-5I_{1}-I_{2}$ & $-\frac{25}{4}I_{1} - \frac{5}{4}I_{2} - \frac{1}{4}I_{3}$ \\
        \hline 
        \textsf{T4} & $-\frac{9}{4}I_{1} + \frac{1}{4}I_{2}$ & $-3I_{1}$ & $-\frac{15}{4}I_{1} - \frac{1}{4}I_{2}$ \\
        \hline 
        \textsf{T5} & $-\frac{7}{2}I_{1}$ & $-5I_{1}$ & $-\frac{13}{2}I_{1}$ \\
        \hline 
        \textsf{T6} & $-\frac{33}{8}I_{1} + \frac{1}{4}I_{2}$ & $-\frac{11}{2}I_{1}$ & $-\frac{55}{8}I_{1} - \frac{1}{4}I_{2}$ \\
        \hline 
        \textsf{T7} & $-\frac{33}{8}I_{1} + \frac{1}{4}I_{2}$ & $-\frac{11}{2}I_{1}$ & $-\frac{55}{8}I_{1} - \frac{1}{4}I_{2}$ \\
        \hline 
        \textsf{T8} & $\alpha_{1} - 3\kappa$ & $\alpha_{1} - 4\kappa$ & $\alpha_{1} - 5\kappa$ \\
        \hline 
        \textsf{T9} & Less than  $I_{1} - 3\kappa$ & $I_{1} - 4\kappa$ & $I_{1} - 5\kappa$ \\
        \hline 
      \end{tabular}
    \end{table}

    \smallskip
    \noindent {\bf \em Find $T_3^*, T_4^*$ and $T_5^*$.} First, we
    decide $T_3^*, T_4^*$ and $T_5^*$ under the constructed joint
    distribution $\bar{P}(x)$ and communication costs $\bar{c}$. Under
    our scenario, from the fact that
    $\delta_d = \frac{\alpha_1 - \alpha_2}{\sqrt{4(3s+q)}}, \kappa =
    \frac{9}{8}sI_1$ and $I_k \cong (2 \delta_d)^{2k}$, we have
    following order of values, (for large $d$)
    \begin{align*}
      \alpha_1 > \alpha_2 > \kappa > I_1 > I_2 > I_3.
    \end{align*}
    Then, it is obvious that these trees must include two specific
    edges in \textsf{T8}, \ie, $(Z_0,Z_1)$ and $(Z_1,Z_2)$, since
    edges in \textsf{T8} have the largest weight compared to all other
    edges.


    We now decide the structure of $T_3^*=(\set{V},\set{E}_3^*)$. From
    the weight values of each type of edges in
    Table~\ref{tbl:weight-diam}, it is clear that the edge set
    $\set{E}_3^*$ woulde be given by:
    \begin{align*}
      \set{E}_3^* = \textsf{T5} \cup \textsf{T6} \cup \textsf{T8},
    \end{align*}
    then its objective value is
    \begin{align*}
      \set{O}^\gamma(T_3^*) &= \textsf{T5} \times q + \textsf{T6} \times 3s + \textsf{T8} \times 2 \cr
      &= (-\frac{7}{2}I_1) \times q + (-\frac{33}{8}I_1 + \frac{1}{4}I_2) \times 3s + (\alpha_1 - 3 \kappa) \times 2 \cr
      &= -(\frac{99}{8}s + \frac{7}{2}q)I_1 + \frac{3}{4}sI_2 + 2\alpha_1 - 6\kappa.
    \end{align*}

    
    For the structure of $T_5^*=(\set{V},\set{E}_5^*)$, it should be
    constructed as follows. The element nodes prefer to be connected
    to a subset node as \textsf{T2} as possible, then to be connected
    as \textsf{T1}. Among the subset nodes, there is a unique center
    subset node and all other subset nodes are connected to the center
    node by \textsf{T4}, where the center subset node is connected to
    $Z_0$, \ie, \textsf{T5}. Finally, two edges in \textsf{T8} are
    included the tree. Now, the lower bound of the objective value is
    given by
    \begin{align*}
      \set{O}^\gamma(T_5^*) &= \textsf{T2} \times 3s + \textsf{T4} \times (q-1) + \textsf{T5} + \textsf{T8} \times 2 \cr
      &\geq -(\frac{45}{4}s + 3q + \frac{9}{4})I_1 - \frac{9}{4}I_2 + 2\alpha_1 - 10\kappa.
    \end{align*}

    We also highlight that the MWST for the weights
    $I(\bar{P}_{i,j}) - 5 \bar{c}_{i,j}^\gamma$ (\ie, without
    constraint on the diameter of the constructed tree) is constructed
    to have a diameter $5$ reveals that
    $\set{O}^\gamma(T_5^*) \geq \bar{\set{O}}^\gamma_5(T_6^*) \geq
    \set{O}^\gamma(T_6^*)$. In consequence, it does not suffice to
    consider $T_k^*, \forall k \geq 6$.


    Finally, we obtain the structure of
    $T_4^*=(\set{V},\set{E}_4^*)$. As mentioned, two edges of
    \textsf{T8}, $(Z_0,Z_1)$ and $(Z_1,Z_2)$, are included in
    $\set{E}_4^*$. As seen in Table~\ref{tbl:weight-diam}, edges of
    \textsf{T1}, \textsf{T2} and \textsf{T4} have larger weight than
    others, we need to contain as many edges of \textsf{T1},
    \textsf{T2} and \textsf{T4} as possible. Therefore, there are two
    candidates for $T_4^*$. The first candidate tree, denoted by
    $T_4^{*,1}=(\set{V}, \set{E}_4^{*,1})$, is the tree where exactly
    $s$ number of subset nodes, denoted by $Y'$, are connected to
    $Z_0$ by \textsf{T5} and each subset node in $Y'$ is connected to
    exactly $3$ element nodes in $X$ by \textsf{T1} or \textsf{T2}. To
    retain a diameter of $4$, remaining $q-s$ number of subset nodes
    in $Y \setminus Y'$ are connected with each other by constructing
    edges in \textsf{T4}. The corresponding objective value is given
    by
    \begin{eqnarray} \label{eq:obj1-diam4}
      \set{O}^\gamma(T_4^{*,1}) &=& (\textsf{T1} ~\text{or}~ \textsf{T2}) \times 3s + \textsf{T4} \times (q-s) + \textsf{T5} \times s + \textsf{T8} \times 2 \cr
      &=& -(11s + 3q)I_1 + 2 \alpha_1 - 8 \kappa.
    \end{eqnarray}

    It is now clear that this optimal tree $T_4^{*,1}$ exists if and
    only if there is a solution of {\bf X3C} problem. In particular,
    finding optimal tree $T_4^{*,1}$ with objective value of
    $-(11s + 3q)I_1 + 2\alpha_1 - 8\kappa$, sub-collection
    $Y' \subset Y$ forms an exact cover of all element nodes in $X$.

    Second candidate tree, denoted by
    $T_4^{*,2}=(\set{V},\set{E}_4^{*,2})$, is the tree where only one
    center subset node is connected to $Z_1$ by \textsf{T9}. Then,
    there exist $3s$ number of edges in \textsf{T2} and $q-1$ number
    of edges in \textsf{T4}. The lower bound of the objective value of
    this tree is given by
    \begin{align} \label{eq:obj2-diam4}
      \set{O}^\gamma(T_4^{*,2}) &\geq \textsf{T2} \times 3s + \textsf{T4} \times (q-1) + \textsf{T9} + \textsf{T8} \times 2 \cr
      &\geq -(9s + 3q -2)I_1 + 2\alpha_1 - 12 \kappa.
    \end{align}
    
    From \eqref{eq:obj1-diam4} and \eqref{eq:obj2-diam4}, we can
    observe that for $\kappa = \frac{9}{8}sI_1$, the optimal solution
    is attained at $T_4^{*,1}$, \ie,
    $\set{O}^\gamma(T_4^{*,1}) >
    \set{O}^\gamma(T_4^{*,2})$. Therefore, when solving the problem
    {\bf SYNC}$({\bm n})$ in \eqref{eq:CDG-learn-sync} to prove
    optimality on an instance constructed scenario of
    $\bar{P}, \bar{c}$, the question of the {\bf X3C} problem can now
    simply be answered, \ie, the problem {\bf SYNC}$({\bm n})$ is
    polynomially reducible to {\bf X3C} problem, which completes the
    proof of Lemma.
  \end{proof}
  The proof of NP-hardness of {\bf SYNC}$({\bm n})$ in
  \eqref{eq:CDG-learn-sync} is a direct consequence of
  Lemma~\ref{lem:red-x3c}.

\end{proof}


 \subsection{Proof of Theorem~\ref{thm:eexponent-async}}
\label{sec:proof-async-alg}

The proof of Theorem~\ref{thm:eexponent-async} is quite similar and
straightforward to that in \cite{tan2011tree}.

\vspace{1cm}



\subsection{Proof of
  Theorem~\ref{thm:eexponent-sync}} \label{sec:proof-sync-alg}


\begin{proof}[Proof.]

  \noindent {\bf \em (i) Crossover rate.}
  We present our proof into following $4$ steps. In {\bf Step 1}, we
  prove the existence of the crossover rate $J_{e,e'}(T)$, and in {\bf
    Step 2}, we show the expression of $J_{e,e'}(T)$ in
  \eqref{eq:crate-sync}. We then prove the existence of the optimizer
  $Q^*$ and show that $J_{e,e'}(T) > 0$ in {\bf Step 3}.

  \smallskip
  \noindent \underline{{\bf Step 1.}} First, we recall the definition
  of the crossover rate:
  \begin{align*}
    J_{e,e'}(T) &= \lim_{n \rightarrow \infty} - \frac{1}{n} \log \mathbb{P}(C_{e,e'}(T)) \cr
    &= \lim_{n \rightarrow \infty} - \frac{1}{n} \log \mathbb{P}\left( \left\{ w^{\tt Sync}_e(T,\hat{P}({\bm x}^{1:n})) \leq w^{\tt Sync}_{e'}(T,\hat{P}({\bm x}^{1:n}))  \right\}  \right).
  \end{align*}
  First, if $|r(T,e) - r(T,e')| \geq |\set{X}| \log |\set{X}|$, the
  crossover event $C_{e,e'}(T)$ is obviously empty set, since
  $|\set{X}| \log |\set{X}|$ is the maximum of the mutual information
  $I_e(\hat{P})$, which results the crossover rate
  $J_{e,e'}(T) = \infty$.

  Now, assuming that $|r(T,e) - r(T,e')| < |\set{X}| \log |\set{X}|$,
  to prove the existence of the limit, we define the set
  $\set{R} \subset \set{P}(\set{X}^4)$ by
  \begin{align*}
    \set{R} \defeq \left\{ Q \in \set{P}(\set{X}^4): w^{\tt Sync}_e(T,Q) \leq w^{\tt Sync}_{e'}(T,Q)  \right\},
  \end{align*}
  then, we can re-express the crossover rate as
  \begin{align*}
    J_{e,e'}(T) = \lim_{n \rightarrow \infty} - \frac{1}{n} \log \mathbb{P}(\hat{P}_{e,e'} \in \set{R}).
  \end{align*}
  If $\set{R}$ is closed and
  $\set{R} = \text{cl}(\text{int}(\set{R}))$, then following is a
  direct consequence of the Sanov's theorem \cite{bucklew90ldp},
  \begin{align*}
    J_{e,e'}(T) &= \lim_{n \rightarrow \infty} - \frac{1}{n} \log \mathbb{P}(\hat{P}_{e,e'} \in \set{R}) = \inf_Q \{ D_{\text{KL}}(Q \parallel P_{e,e'}): Q \in \set{R} \}.
  \end{align*}
  It suffices to shat that $\set{R}$ is closed and
  $\set{R} = \text{cl}(\text{int}(\set{R}))$, to complete the proof of
  the existence of the limit. First, for a fixed tree $T$, we define a
  function
  $h(Q) \defeq w^{\tt Sync}_{e'}(T,Q) - w^{\tt Sync}_e(T,Q)$, which is
  a continuous function because the mapping $Q \mapsto Q_e$ and
  $I(Q_e)$ are continuous for any edge $e$, and the term $r(T,e)$ is
  constant on a fixed $e$. Therefore, $\set{R}$ is a closed set
  because it is an inverse image of a closed set by continuous
  function $h(Q)$. Second, it is trivial that
  $\set{R}' \defeq \{ Q \in \set{P}(\set{X}^4): w^{\tt Sync}_e(T,Q) <
  w^{\tt Sync}_{e'}(T,Q) \}$ is a subset of $\text{int}(\set{R})$ from
  the fact that $h(Q)$ is continuous. Now, we choose an arbitrary
  distribution $M \in \set{R} \setminus \set{R}'$, where we can
  regulate either $I_e(M)$ or $I_{e'}(M)$ very slightly while
  maintaining the mutual information of the other edge, \ie,
  $I_e(M') = I_e(M) - \delta$ and $I_{e'}(M') = I_{e'}(M) + \delta$
  for sufficiently small positive number $\delta$. This distribution
  $M'$ is included in $\set{R}'$, since it still satisfies
  $w^{\tt Sync}_e(T,M') < w^{\tt Sync}_{e'}(T,M')$. We can always
  specify $M' \in \set{R}'$ which converges to $M$ as
  $\delta \rightarrow 0$, which concludes that $\set{R}$ is the
  closure of $\set{R}'$. Finally, from the Sanov's theorem, we have
  the following expression of the crossover rate:
  \begin{eqnarray} \label{eq:crate-sync-1}
    J_{e,e'}(T) = \inf_{Q \in \set{P}(\set{X}^4)} \{ D_{\text{KL}}(Q \parallel P_{e,e'}) : w^{\tt Sync}_e(T,Q) \leq w^{\tt Sync}_{e'}(T,Q) \}.
  \end{eqnarray}

  \smallskip
  \noindent \underline {\bf Step 2.} In {\bf Step 2}, we show that if
  the optimal solution of \eqref{eq:crate-sync-1}, denoted by $Q^*$,
  exists, then it is attained when
  $w^{\tt Sync}_e(T,Q^*) = w^{\tt Sync}_{e'}(T,Q^*)$. We prove this by
  contradiction. Suppose there is an optimal distribution $M^*$ which
  minimizes the $D_{\text{KL}}(M^{*}_{e,e'} \parallel P_{e,e'})$ and
  satisfies $W(M^{*}_{e'})>W(M^{*}_{e})$. For $\lambda \in [0,1],$
  consider
  $M^\lambda_{e,e'} := (1-\lambda) M_{e,e'}^* + \lambda P_{e,e'}.$ As
  $\lambda$ increases from $0$ to $1$, $h(M_{e,e'}^\lambda)$ moves
  from $h(M_{e,e'}^*)$ to $h(P_{e,e'})$ where $h(P_{e,e'})$ must be
  smaller than $0$. $h(M_{e,e'}^\lambda)$ is a continuous function
  with respect to $\lambda$. Thus, there must be $\lambda \in (0,1)$
  such that $h(M_{e,e'}^\lambda)$ equals to $0$ and
  $M_{e,e'}^{\lambda} \in R \setminus R'$.
  We can use the convexity of the KL-divergence to prove the
  contradiction.
  \begin{equation}
    \begin{split}
      & D_{\text{KL}}(M^{\lambda}_{e,e'} \parallel P_{e,e'}) = D_{\text{KL}}\big( (1-\lambda)M_{e,e'}^{*}+ \lambda \cdot P_{e,e'} \parallel P_{e,e'} \big) \\
      & \quad \leq (1-\lambda)D_{\text{KL}}(M_{e,e'}^{*} \parallel P_{e,e'})+ \lambda \cdot D_{\text{KL}}(P_{e,e'} \parallel P_{e,e'}) \\
      & \quad \quad (\because \text{Convexity of the KL-divergence}) \\
      & \quad = (1-\lambda)D_{\text{KL}}(M_{e,e'}^{*} \parallel P_{e,e'})
       < D_{\text{KL}}(M_{e,e'}^{*} \parallel P_{e,e'}).
    \end{split}
  \end{equation}
  Therefore, we can derive the contradiction of the assumption about
  the existence of $M^{\lambda}_{e,e'}$ which is an element of
  $R \setminus R'$ and has smaller KL-divergence than
  $M^{*}_{e,e'}$. The conclusion in \textbf{Step 2}
  is \begin{equation} \label{eq:infRinfR'} \inf _{M \in R}
    D_{\text{KL}}(M_{e,e'} \parallel P_{e,e'})= \inf _{M \in R
      \setminus R'} D_{\text{KL}}(M_{e,e'} \parallel P_{e,e'}).
  \end{equation}

  \smallskip
  \noindent \underline{\bf Step 3.} Continuing from \textbf{Step 2},
  we should show the existence of the minimizer $Q^{*}_{e,e'}$. If we
  can prove the compactness of $R \setminus R'$, we can get the
  existence of the minimizer $Q^{*}_{e,e'}$ in $R \setminus R'$ from
  Weierstrass' extreme value theorem. Therefore, by combining the
  result of \textbf{Step 2} or the equation (\ref{eq:infRinfR'}), we
  can derive the existence of the minimizer $M^{*}_{e,e'}.$ To prove
  the compactness of the $R \setminus R'$, we exploit Heine-Borel
  theorem and show the $R \setminus R'$ is bounded and closed. The
  boundedness is obvious because
  $\mathcal{P}(\mathcal{X}^{4}) \subset
  [0,1]^{|\mathcal{X}|^{4}}$. The closedness is also obvious from the
  fact that $R \setminus R'=h^{-1}(\{0\})$.

  Finally, we need to prove $J_{e,e'}(T)>0$. Use
  contradiction. Suppose $J_{e,e'}(T)=0$. That means
  $$\inf_{Q \in \set{P}(X^{4})} \{D_{\text{KL}}(Q \parallel P_{e,e'}): W(Q_{e})=W(Q_{e'})\}=0.$$ Also in step 3, we find the existence of the minimizer $Q^{*}_{e,e'}$. That means
  $D_{\text{KL}}(Q^{*}_{e,e'} \parallel P_{e,e'})=0$, so
  $Q^{*}_{e,e'} \equiv P_{e,e'}$ and $W(P_{e})=W(P_{e'}).$
  $W(P_{e})=W(P_{e'})$ is a contradiction from the assumption
  $W(P_{e})>W(P_{e'}).$

  Consequently, from {\bf Step 1}, {\bf Step 2} and {\bf Step 3}, we
  complete the proof of Theorem~\ref{thm:eexponent-sync} {\em (i)}.
  	



  \smallskip
  \noindent {\bf \em (ii) Error exponent.}
  We first get
  \begin{eqnarray} \label{eq:aunion-sync} \mathcal{A}_n(\gamma)
    \subset \bigcup_{t=1}^{d-1} \bigcup _{e' \in
      \set{E}_c(\textsf{T}_t) \setminus \textsf{e}_t}
    C_{\textsf{e}_t,e'}(\textsf{T}_{t}).
  \end{eqnarray}
  Then, we will prove that for each step $t$ of algorithm
  ${\tt LearnSync}$, the below is correct,
  \begin{align} \label{eq:ceventprob-sync}
    \mathbb{P}(C_{\textsf{e}_t,e'}(\textsf{T}_t)) \leq \binom{n-1+
      |\mathcal{X}|^{4}}{|\mathcal{X}|^{4}-1}\exp(-n \cdot
    J_{\textsf{e}_t,e'}(\textsf{T}_{t})).
  \end{align}
  We can get the result \eqref{eq:sbound-sync} from the equations
  \eqref{eq:aunion-sync} and \eqref{eq:ceventprob-sync}.


  We can regard $\mathcal{A}_{n}(\gamma)$ be a union of the events
  that the SYNC algorithm finds the wrong tree-structure rather than
  the true tree-structure $T$ at the final step $d$ under the $n$
  samples. To find the wrong tree by the SYNC algorithm, we must find
  wrong edge $e'$ among $\set{E}_c(T_t) \setminus e_{t}$ which is the
  set of the possible edges at some step $t$ for
  $t=1,...,d-1$. Therefore we can easily get the eq
  (\ref{eq:aunion-sync}). Unfortunately, we can not tell this union of
  the events exactly equals to $\mathcal{A}_{n}(\gamma)$. This is
  because we can choose the edges in kinds of reverse order. For
  example, we may choose $e_{t+1}$ at the step $t$ and choose $e_{t}$
  at the step $t+1$. Although this translation does not affect after
  the step $t+2$ or the final step, it is regarded as an error and
  included in $C_{e_{t},e_{t+1}}(T_{t})$ and
  $C_{e_{t+1},e_{t}}(T_{t})$ which are calculated as an error event.

  \smallskip
  \noindent {\bf \em (iii) Error probability.}
  We derive the \eqref{eq:ceventprob-sync} from the result of (finite
  domain) Sanov's theorem in \cite{dembo2010ld}
  \begin{equation} \label{eq:sanovLn} \mathbb{P}(C_{e_{t},e'}(T_{t}))
    \leq |\mathcal{L}_{n}|\exp(-n \cdot J_{e_{t},e'}(T_{t})),
  \end{equation}
  where $|\mathcal{L}_n|$ represents the number of the possible
  empirical distributions
  $\hat{P}_{e,e'} \in \mathcal{P}(\mathcal{X}^4)$ by $n$ samples. In
  \cite{dembo2010ld}, they use $(n+1)^{|\set{X}^4|}$ as upper bound of
  $|\set{L}_n|$. Although this bound is enough to prove the aysmptotic
  property of Sanov's theorem, it is still very loose bound. We find
  exact value as
  $|\mathcal{L}_{n}|= \binom{n-1+
    |\mathcal{X}|^{4}}{|\mathcal{X}|^{4}-1}. $ Suppose that you order
  $|\mathcal{X}|^{4}-1$ number of black balls and $n$ number of white
  balls in a row. The white balls would be divided into
  $|\mathcal{X}|^{4}$ partitions by the black balls. Let
  $a_{1},a_{2},...,a_{|\mathcal{X}|^{4}}$ be the number of the white
  balls in the each partitions by the black balls. It is obvious that
  $\sum_{i=1}^{|\mathcal{X}|^{4}} a_{i}=n$ and we can regard
  $(\frac{1}{n}a_{1},\frac{1}{n}a_{2},...,\frac{1}{n}a_{|\mathcal{X}|^{4}})$
  as an empirical distribution by $n$ samples on $\mathcal{X}^{4}.$
  Therefore, the number of the combinations for the black balls and
  the white balls
  $\binom{n-1+ |\mathcal{X}|^{4}}{|\mathcal{X}|^{4}-1}$ equals to
  $|\mathcal{L}_{n}|.$ By substituting $|\mathcal{L}_{n}|$ into the eq
  (\ref{eq:sanovLn}), we get
  \begin{align}
    \mathbb{P}(C_{e_{t},e'}(T_{t})) &\leq  \binom{n-1+ |\mathcal{X}|^{4}}{|\mathcal{X}|^{4}-1} \exp(-n \cdot J_{e_{t},e'}(T_{t})) \cr &\leq \binom{n-1+ |\mathcal{X}|^{4}}{|\mathcal{X}|^{4}-1} \exp(-n \cdot K(\gamma)).
  \end{align} 
  From the fact that, we need to choose an edge among the edges
  between $t$ nodes in $T_{t}$ and $d-t$ nodes.  At the step
  $t$ under $T_{t}$, we have $t(d-t)$ candidates of edges in
  $\set{E}_c(T_t)$. We finally get
  \begin{equation}
    \begin{split}
      \mathbb{P}(\mathcal{A}_{n}(\gamma)) & \leq \sum_{t=1}^{d-1} \sum_{e' \in \set{E}_c(T_t) \setminus e_{t}} \mathbb{P}(C_{e_{t},e'}(T_{t})) \cr  & \leq  \left(\sum_{t=1}^{d-1} t(d-t) \right)\binom{n-1+ |\mathcal{X}|^{4}}{|\mathcal{X}|^{4}-1} \exp(-n K(\gamma))  \\
      & = \frac{(d-1)d(d+1)}{6}\binom{n-1+
        |\mathcal{X}|^{4}}{|\mathcal{X}|^{4}-1} \exp(-n K(\gamma)).
    \end{split}
  \end{equation}
  The lower bound for the error exponent rate for
  $\mathbb{P}(\mathcal{A}_{n}(\gamma))$ can be derived as
  \begin{align*}
    \liminf_{n \rightarrow \infty} -\frac{1}{n} \log \mathbb{P}(\mathcal{A}_{n}(\gamma)) \geq K(\gamma).
  \end{align*}

  We highlight that as $n \rightarrow \infty$, the resulting tree
  structure learned by SYNC algorithm equals to the {\em sub-optimal}
  cost efficient tree $\tilde{P}^{\gamma}(x)$ with respect to
  $\tilde{T}(\gamma) = (\set{V}, \tilde{\set{E}}(\gamma))$, with
  probability $1$.

\end{proof}

\end{document}